\documentclass[a4paper,10pt]{article}
\usepackage[left=2.5cm, right=2.5cm, top=3cm, bottom=3cm]{geometry}
\usepackage[square]{natbib}

\usepackage{microtype}
\usepackage{graphicx}
\usepackage{subfigure}
\usepackage{amsmath}
\usepackage{amsfonts}
\usepackage{amssymb}
\usepackage{amsthm}
\usepackage{booktabs}
\usepackage{hyperref}

\usepackage[affil-it]{authblk}
\usepackage{subfigure,units}
\usepackage[capitalise,noabbrev]{cleveref}
\usepackage{xcolor}
\usepackage{algorithm}
\usepackage{algorithmic}
\newtheorem{prop}{Proposition}
\theoremstyle{remark}
\newtheorem*{remark}{Remark}
\renewcommand{\algorithmicrequire}{\textbf{Input:}}
\renewcommand{\algorithmicensure}{\textbf{Output:}}

\newlength{\textfloatsepsave} 

\begin{document}

\title{Log-Euclidean Signatures for Intrinsic Distances Between Unaligned Datasets}

\author[1]{Tal Shnitzer \thanks{Corresponding author. Email: talsd@mit.edu}}
\author[2]{Mikhail Yurochkin}
\author[2]{Kristjan Greenewald}
\author[1]{Justin Solomon}

\affil[1]{MIT}
\affil[2]{MIT-IBM Watson AI Lab}

\date{}


\vskip 0.3in

\maketitle

\begin{abstract}
    The need for efficiently comparing and representing datasets with unknown alignment spans various fields, from model analysis and comparison in machine learning to trend discovery in collections of medical datasets.
    We use manifold learning to compare the intrinsic geometric structures of different datasets by comparing their diffusion operators, symmetric positive-definite (SPD) matrices that relate to approximations of the continuous Laplace-Beltrami operator from discrete samples.
    Existing methods typically assume known data alignment and compare such operators in a pointwise manner.
    Instead, we exploit the Riemannian geometry of SPD matrices to compare these operators and define a new theoretically-motivated distance based on a lower bound of the log-Euclidean metric.
    Our framework facilitates comparison of data manifolds expressed in datasets with different sizes, numbers of features, and measurement modalities.
    Our log-Euclidean signature (LES) distance recovers meaningful structural differences, outperforming competing methods in various application domains.
\end{abstract}

\section{Introduction}\label{sec:Intro}
With the increased availability of data-driven methods, the ability to compare data collections has become crucial.
In many cases, such comparisons are impossible without computationally-expensive alignment and manipulation, especially when the datasets are of different dimensions or are embedded in different domains.
This problem is prominent in learning, where the need for comparing datasets arises, e.g., in the evaluation and comparison of neural network (NN) architectures \citep{kornblith2019similarity}, data manifolds of generative models \citep{heusel2017gans}, and tasks in transfer and meta learning \citep{achille2019task2vec}.

A common approach to dataset comparison assumes that complex high-dimensional data lie on some manifold of lower dimension, i.e., the manifold hypothesis \citep{fefferman2016testing}.
Manifold learning techniques exploit this assumption to provide geometrically-motivated intrinsic data representations, commonly by constructing discrete approximations of the underlying manifold's Laplace-Beltrami operator or heat kernel \citep{belkin2003laplacian,coifman2006diffusion}.
Comparing such representations usually requires point-to-point correspondences, yielding an NP-hard quadratic assignment problem \citep{cela2013quadratic}.

One approach to avoiding the need for data alignment is via spectral properties of the operators.
\citet{tsitsulin2019shape} compare unaligned datasets by computing distances between approximations of the heat kernel traces at different temperature scales, lower bounding an optimal transport-based metric computing a pointwise distance between heat kernels \citep{memoli2011spectral}.
This pointwise comparison, providing the theoretical motivation for their distance as in many related methods, overlooks the symmetric positive-definite (SPD) structure of the heat kernel approximation used in practice.
Prior studies show that using appropriate metrics on the space of SPD matrices, forming a differentiable Riemannian manifold, significantly outperforms pointwise Euclidean comparisons  \citep{pennec2006riemannian,arsigny2006log,barachant2013classification}.

In this work, we take into account the SPD structure of heat kernel approximations and use the log-Euclidean (LE) metric \citep{arsigny2006log} to compare them.
We propose the \emph{log-Euclidean signature (LES) distance} for unaligned datasets based on a lower bound of this metric.
Our distance function can be formulated in terms of efficient descriptors, which capture the underlying intrinsic geometry of the data manifold, facilitating the comparison of heterogeneous datasets. 
We demonstrate properties of the LES distance on various data types, including $3$-$4$D point clouds of geometric shapes, single cell RNA-sequencing data, and NN embeddings, where we 
show that the LES distance provides useful insights on architectures and embedding structures. 

\section{Related Work}
Many frameworks for comparing datasets have been proposed.
We outline the most related approaches proposing geometrically-motivated measures. 

\textbf{Geometrically-motivated dataset descriptors.} 
An early work proposing a geometric global fingerprint of a dataset is `Shape-DNA' \citep{reuter2006laplace}, constructed from approximate Laplace-Beltrami eigenvalues. 
This signature was formulated for $2D$ and $3D$ manifolds, and its extension to high-dimensional data is nontrivial.
More recent works include the Geometry Score (GS) \citep{khrulkov2018geometry}, the Network Laplacian Spectral Descriptor (NetLSD), and the Intrinsic Multi-scale Distance (IMD) \citep{tsitsulin2018netlsd,tsitsulin2019shape}.
The GS characterizes datasets by their topological properties using persistent homology, which captures multi-scale structural differences of the data manifolds but comes with high computational costs.
NetLSD and IMD propose dataset signatures based on the heat kernel trace at different temperature scales, for graphs (NetLSD) and for general datasets (IMD).
IMD is the most related work to ours, discussed in Subsection \ref{sub:lesd} and Section \ref{sec:results}.

\textbf{Distances between datasets.}
Methods for computing distances between datasets often estimate some alignment between them. For example, \citet{peyre2016gromov} introduce a technique for computing the Gromov-Wasserstein (GW) distance between unaligned kernels.
Such methods typically lead to non-convex optimization problems with extremely high computational costs.
Distance measures that do not rely on data alignment include the Maximum Mean Discrepancy (MMD) \citep{gretton2012kernel}, Manifold Topology Divergence (MTD) \citep{barannikov2021manifold} and the Cross Local Intrinsic Dimensionality (CrossLID) \citep{barua2019quality}.
MTD measures topological differences of datasets by constructing simplicial complexes between them and evaluating the obtained topology. 
%Similarly to GS, it captures only global topological properties, unlike IMD and our proposed method. 
CrossLID compares data distributions in small neighborhoods of the two datasets.
The MMD, MTD and CrossLID require that datasets are embedded in the same space with the same dimensionality.

Finally, the Fr\'{e}chet Inception Score (FID) \citep{heusel2017gans} computes the Wasserstein-2 distance for Gaussians between outputs of an Inception convolutional NN layer \citep{szegedy2015going}. Due to the Inception layer, the FID score and its many variants are not suited for the evaluation of general datasets, which are the focus of this paper.

For the pairwise measures above, every coupling must be evaluated separately, yielding high computational costs when comparing many datasets. 
In contrast, our framework, while derived from a pairwise distance, first embeds the datasets, allowing for simple Euclidean comparisons.

\section{Preliminaries\label{sec:preliminaries}}
We briefly describe relevant background for the proposed method, first describing a method for approximating the Laplace-Beltrami operator and heat kernel from high-dimensional data, with minimal prior assumptions on the data type. 
We then outline the LE metric, which defines a Riemannian manifold on the space of SPD matrices.

\subsection{Diffusion Operator Approximations from Data\label{sub:dm}}
Manifold learning techniques typically construct positive-semidefinite kernel matrices from sampled data. 
We focus here on diffusion maps, a manifold learning technique that defines the diffusion operator, which was shown to have desirable convergence properties \citep{coifman2006diffusion}.

Given a set of points $x_i\in\mathbb{R}^d$, $i=1,\dots,N$, assumed to lie on some low-dimensional manifold $\mathcal{M}$ embedded in the ambient space $\mathbb{R}^d$, an affinity matrix is constructed by:
\begin{equation}
    K_{i,j}=\exp\left(-\frac{d^2(x_i,x_j)}{\sigma^2}\right),\label{eq:affkern}
\end{equation}
where $d(\cdot,\cdot)$ denotes some notion of distance suitable to the data (e.g., Euclidean distance) and $\sigma$ is the kernel scale, typically the median distance between data points multiplied by some constant \citep{ding2020phase}.
The affinity matrix is then normalized, producing the diffusion operator $\mathbf{W}_{DM}$:
\begin{eqnarray}
\tilde{\mathbf{W}} = \tilde{\mathbf{D}}^{-1}\mathbf{K}\tilde{\mathbf{D}}^{-1}, & \tilde{D}_{i,i} = \sum_{j=1}^N K_{i,j},\label{eq:dmnorm}\\
\mathbf{W}_{DM} = \mathbf{D}^{-1}\tilde{\mathbf{W}}, & D_{i,i} = \sum_{j=1}^N \tilde{W}_{i,j}\, ,\label{eq:dmkern}
\end{eqnarray}
where the normalization in \eqref{eq:dmnorm} mitigates the effect of non-uniform data sampling \citep{coifman2006diffusion}.

The diffusion operator satisfies several convergence properties. 
Specifically, $\mathbf{W}_{DM}^{t/\sigma^2}$ converges pointwise to the Neumann heat kernel of the underlying manifold as $N\rightarrow\infty$ and $\sigma^2\rightarrow 0$, and $\mathbf{L}=(\mathrm{\mathbf{I}}-\mathbf{W}_{DM})/\sigma^2$ converges to the Laplace-Beltrami operator \citep{coifman2006diffusion}.
Both continuous operators capture geometric properties of the manifold.
In addition, eigenvalues and eigenvectors of $\mathbf{L}$ converge to Laplace-Beltrami eigenvalues and eigenfunctions \citep{belkin2007convergence,dunson2021spectral}.

The diffusion operator has proven to be extremely useful in numerous tasks and diverse applications, providing data-driven low-dimensional representations that preserve the structure of the underlying manifold, see e.g.\  \citep{talmon2013diffusion,mishne2012multiscale,talmon2013empirical}.

Following previous work \citep{coifman2014diffusion,katz2020spectral,shnitzer2022spatio}, we define an SPD matrix that is similar to $\mathbf{W}_{DM}$ given by:
\begin{equation}
    \mathbf{W}=\mathbf{D}^{1/2}\mathbf{W}_{DM}\mathbf{D}^{-1/2}.\label{eq:dm_spd}
\end{equation}
Since $\mathbf{W}$ and $\mathbf{W}_{DM}$ are similar, they share eigenvalues, and their eigenvectors are related by $\mathbf{D}^{1/2}$.
Our proposed distance function relies on the eigenvalues of $\mathbf{W}$. 
The convergence properties of $\mathbf{W}_{DM}$ and its similarity to $\mathbf{W}$ thus imply that our distance compares (a transformation of) discrete approximations of the heat kernel spectrum.

\subsection{Riemannian Manifold of SPD Matrices\label{sub:lem}}
The space of $N\times N$ SPD matrices forms a Riemannian manifold when endowed with an appropriate metric. 
The most common metrics are the affine-invariant (AI) \citep{pennec2006riemannian} and LE metrics \citep{arsigny2006log}, both providing closed-form geodesic distances.
We focus on the LE metric, since it shares useful properties with the AI metric---i.e., invariance to inversion, translation, scaling and orthogonal transformations---while requiring lower computational cost. 

Let $\mathrm{Sym}^+_N$ denote the set of $N\times N$ real SPD matrices, and let $\mathrm{Sym}_N$ denote the set of $N\times N$ real symmetric matrices.
Define the logarithmic product for $\mathbf{W}_1,\mathbf{W}_2\in\mathrm{Sym}^+_N$ as $\mathbf{W}_1\odot\mathbf{W}_2=\exp(\log\mathbf{W}_1 + \log\mathbf{W}_2)$, where $\exp$ and $\log$ denote the matrix exponent and logarithm functions. 
This defines a Lie group structure on $\mathrm{Sym}^+_N$. By commutativity of the logarithmic product, the Lie group can be equipped with a bi-invariant Riemannian metric induced by any inner product on $\mathrm{Sym}_N=\mathcal{T}_{\mathbf{I}}\mathrm{Sym}^+_N$, the tangent space at the identity (Lie algebra) \citep{arsigny2007geometric}; the associated geodesic distance is then $d(\mathbf{W}_1,\mathbf{W}_2)=\left\Vert\log\mathbf{W}_1-\log\mathbf{W}_2\right\Vert$, where the norm is induced by the inner product.
The $L^2$ inner product (and corresponding Frobenius norm) yields the following LE distance:
\begin{eqnarray}
    d_{\mathrm{LE}}\left(\mathbf{W}_1,\mathbf{W}_2\right)^2 & = & \left\Vert\log\mathbf{W}_1-\log\mathbf{W}_2\right\Vert_F^2\, .\label{eq:lem}
\end{eqnarray}
This metric is invariant to similarity transformations (scaling and orthogonal transformations).

From the SPD manifold perspective, the matrix log in the LE distance maps SPD matrices onto the tangent space at the identity, i.e., the space of symmetric matrices, which is closed under Euclidean operations. 
Therefore, the distances between SPD matrices are computed as the Euclidean distance between their mappings at the identity, reducing the space of SPD matrices to a flat space \citep{arsigny2007geometric}. 

The LE metric improves performance in various applications involving SPD matrices, such as diffusion tensor imaging \citep{arsigny2006log}, image set classification \citep{huang2015log}, and land cover classification \citep{guan2019covariance}. 
Computing LE distances, however, requires pointwise correspondence of the two SPD matrices; such data alignment may be expensive or impossible to obtain. 
One of the key contributions of our work is to alleviate the need for data alignment via a lower bound on the LE distance. 
Therefore, the choice of the LE metric is intentional: it is simpler relatively to other SPD metrics, while performing well in practice.
Overcoming the need for alignment in other metrics may be more challenging, e.g.\ due to matrix multiplication/inversion in the AI metric \citep{pennec2006riemannian}.
In Appendix \ref{appsub:tori_addres} we empirically compare the LE metric, the AI metric, and Euclidean distance to motivate the use of the SPD Riemannian geometry and specifically the LE metric.
% \new{This relates to a key advantage of the LE metric over others, cf.\ \citep{pennec2006riemannian,lin2019riemannian,bhatia2019bures}, in our work. Since we address unaligned datasets, a key property of the LE metric is its simplicity, which we use to motivate a new distance without need for pointwise correspondence.
% Overcoming the need for alignment in the context of other metrics may be more challenging, e.g.\ due to matrix multiplication/inversion in the AI metric \citep{pennec2006riemannian}.
% As an additional motivation for using the SPD Riemannian geometry and specifically the LE metric, Appendix \ref{appsub:tori_addres} empirically compares the LE metric, the AI metric, and Euclidean distance.}

\section{Distance Between Intrinsic Representations of Unaligned Data\label{sec:method}}
Based on the tools in Section \ref{sec:preliminaries}, we propose a distance measure for comparing the underlying manifold geometry of unaligned datasets.
Our method conceptually consists of two steps. 
First, we represent each dataset using an SPD diffusion operator capturing underlying geometry of the data, constructed in Subsection \ref{sub:dm}.
Second, we account for the Riemannian structure of the space of SPD matrices by defining our new theoretically motivated distance measure for unaligned datasets as a lower bound of the LE metric.
The resulting distance is based on the ordered eigenvalues of the matrices representing the data. 
We refer to this distance as the \emph{LES (log-Euclidean signature) distance}.
This procedure embeds whole datasets into a new space based on their underlying geometry, allowing us to compare datasets of \emph{different sizes}, with \emph{different numbers of features}, and from \emph{various measurement modalities}.

Below, we lay out the details and derivations leading to the LES distance  and describe implementation considerations for large datasets.
Algorithm \ref{alg:lesd} summarizes our approach.

\subsection{Bounding the Log-Euclidean Distance\label{sub:lb}}
We begin with a setting involving datasets of the same size, but relax this assumption  in Subsection \ref{subsub:spect_est}. 
Let $\mathbf{W}_1,\mathbf{W}_2\in\mathbb{R}^{N\times N}$ denote SPD diffusion operators, constructed according to \eqref{eq:dm_spd} from two datasets.
We aim to compare these matrices in a manner that acknowledges their SPD structure using the LE distance \eqref{eq:lem}.
Such a comparison, however, requires point-to-point correspondences.

Existing methods for aligning and comparing datasets typically come with extremely high computational costs.
Therefore, we propose a different approach for comparing the matrix representations of the datasets through a lower bound of the LE distance.
The lower bound is described in the following proposition, along with an upper bound. 
\begin{prop}\label{prop:lelb}
Given two SPD matrices, $\mathbf{W}_1,\mathbf{W}_2\in\mathbb{R}^{N\times N}$, the following holds:
\begin{flalign}
    & \left\Vert\log\mathbf{W}_1\!-\!\log\mathbf{W}_2\right\Vert_F^2 \!\geq\! \sum_{i=1}^N\left(\log\lambda^{(1)}_i\!-\!\log\lambda^{(2)}_i\right)^2\label{eq:lemlb}\\
    & \left\Vert\log\mathbf{W}_1\!-\!\log\mathbf{W}_2\right\Vert_F^2 \!\leq\! \sum_{i=1}^N\left(\log\lambda^{(1)}_i\!-\!\log\lambda^{(2)}_{N-i+1}\right)^2\nonumber
\end{flalign}
where $\lambda^{(\ell)}_i$ denotes the $i$'th eigenvalue of $\mathbf{W}_\ell$ and the eigenvalues of both matrices are organized in decreasing order, i.e., $\lambda^{(\ell)}_1\geq\lambda^{(\ell)}_2\geq\dots\geq\lambda^{(\ell)}_N$.
\end{prop}
\begin{proof}
The proof follows directly from Theorem III.4.4 in \citet{bhatia2013matrix}, which states that for any pair of Hermitian matrices $\mathbf{A},\mathbf{B}\in\mathbb{R}^{N\times N}$ and any symmetric gauge function $\Phi$ on $\mathbb{R}^N$, we have $\Phi\left(\lambda^{(\mathbf{A})\downarrow}-\lambda^{(\mathbf{B})\uparrow}\right)\geq\Phi\left(\lambda^{(\mathbf{A}-\mathbf{B})}\right)\geq\Phi\left(\lambda^{(\mathbf{A})\downarrow}-\lambda^{(\mathbf{B})\downarrow}\right)$, where $\lambda^{(\mathbf{A})\downarrow}$ denotes the eigenvalues of $\mathbf{A}$ organized in decreasing order and $\lambda^{(\mathbf{B})\uparrow}$ denotes the eigenvalues of $\mathbf{B}$ organized in increasing order.
This theorem holds due to the symmetry of $\log\mathbf{W}_\ell$ and since the $l^2$ norm is a symmetric gauge function.
\end{proof}
Equality for the lower bound is obtained when the matrices are mutually diagonalizable and have an \emph{identical} order of basis elements w.r.t.\ the ordered eigenvalues. 
Equality for the upper bound is obtained when the matrices are mutually diagonalizable and have an \emph{opposite} order of basis elements.

Both bounds in Proposition \ref{prop:lelb} are easily computable from the matrices $\mathbf W_\ell$.
However, due to the reversed order of the eigenvalues, the upper bound is loose and is not suitable in practice due to numerical limitations on the estimation of small eigenvalues.
We construct LES from the \emph{lower} bound because the eigenvalues of the matrices are ordered consistently in the sum. 
In particular, the right-hand side of \eqref{eq:lemlb} is a squared Euclidean distance between vectors whose elements are $\log\lambda^{(\ell)}_i$.
This choice is further motivated by recent distances based on lower bounds of other common metrics \citep{tsitsulin2019shape, barannikov2021manifold}.

\subsection{Accommodating Large Datasets of Different Sizes\label{sub:bigdata}}
We use the lower bound from \eqref{eq:lemlb} to define our proposed distance, presented later in Subsection \ref{sub:lesd}, Eq. \eqref{eq:lesdist}.
However, a few computational considerations are required in its design to compare large datasets with different numbers of samples.
First, we add kernel regularization to $\mathbf{W}$ and provide theoretical reasoning in Subsection \ref{subsub:opreg}.
Second, we address datasets of different sizes by truncating the full spectrum to a fixed length $K$.
Third, we describe methods for estimating the spectrum and the matrix $\mathbf{W}$ in \eqref{eq:dm_spd} to facilitate the use of the LES distance in large datasets. 
The latter two points are discussed in Subsection \ref{subsub:spect_est}.

\subsubsection{Operator regularization\label{subsub:opreg}}
As the number of data points increases, the diffusion operator in \eqref{eq:dmkern} converges to the heat kernel.
The matrix logarithm, however, is bounded only for finite matrices, since as $N\rightarrow\infty$, the eigenvalues decay and $\lim_{k\rightarrow\infty}\log\lambda_k^{(\ell)}=-\infty$ \citep{quang2014log}.
This is especially prominent due to numerical instabilities and errors, causing zero or even negative eigenvalues in large matrices that are expected to be SPD.
Common practice adds kernel regularization of the form $\mathbf{W}_\ell+\gamma\mathrm{\mathbf{I}}$, where $0<\gamma\in\mathbb{R}$ and $\mathrm{\mathbf{I}}$ denotes the identity. 

Regularizing the operators serves two purposes. First, it allows us to compute the matrix logarithm for matrices whose eigenvalues decrease below the minimal numerical error limit. Second, it reduces the effect of noise, since noise commonly governs the smaller eigenvalues (for relatively high signal-to-noise ratios); in particular, this regularization reduces the effect of eigenvalues smaller than $\gamma$ in $\log(\lambda^{(\ell)}+\gamma)$.
Hence, this regularization can be thought of as truncating the log-spectrum at a specific precision.

This regularization leads to theoretical complications, however, since the infinite-dimensional identity operator is unbounded in the Hilbert-Schmidt norm, which extends the Frobenius norm to infinite dimensions.
\citet{quang2014log,minh2016approximate} tackle this issue by generalizing the LE metric to an infinite-dimensional setting via positive definite unitized Hilbert-Schmidt operators. 
They define the Log-Hilbert-Schmidt metric for regularized positive operators, which is bounded for both finite and infinite operators. 
In infinite dimensions, they modify the LE metric as follows:
\begin{flalign}
    d^2_{\mathrm{logHS}}\left(\mathbf{W}_1+\gamma\mathrm{\mathbf{I}},\mathbf{W}_2+\mu\mathrm{\mathbf{I}}\right) = \left(\log\gamma\!-\!\log\mu\right)^2
    +\!\left\Vert\log\left(\nicefrac{\mathbf{W}_1}{\gamma}\!+\!\mathrm{\mathbf{I}}\right)\!-\!\log\left(\nicefrac{\mathbf{W}_2}{\mu}\!+\!\mathrm{\mathbf{I}}\right)\!\right\Vert_{HS}^2
    \!\!\!\!\!\!\!
\end{flalign}
where $\left\Vert\cdot\right\Vert_{HS}^2$ denotes the Hilbert-Schmidt norm, $0<\gamma,\mu\in\mathbb{R}$, and $\mathbf{W}_1+\gamma\mathrm{\mathbf{I}}$ and $\mathbf{W}_2+\mu\mathrm{\mathbf{I}}$ are regularized positive definite operators.
When $\gamma=\mu$ and in the finite case, this formulation coincides with the standard LE metric applied to regularized operators. 
In this work, we follow \citet{minh2016approximate} and set $\gamma=\mu$.
This derivation motivates the use of the regularized diffusion operators in the LE metric, guaranteeing that the defined distance will not diverge as the number of samples, $N$, increases.

\begin{remark}
Another way to address the numerical limitations, i.e.\ eigenvalues that decrease below the minimal numerical error limit, is with a fixed-rank framework for symmetric positive semidefinite matrices.
In fact, such a framework can give additional justification for our distance. 
Our \emph{exact} distance can also be obtained by lower-bounding a fixed-rank framework in the spirit of \citet{bonnabel2010riemannian,collard2014anisotropy}, by zeroing the term related to the rotation group (see Eq (49) and Alg.\ 3 in \citet{collard2014anisotropy}), which requires the alignment.
\end{remark}
\begin{remark}
The choice of the regularization parameter $\gamma$ can be motivated by expected noise levels in the data.
For very noisy data, the spectrum components that represent noise are expected to correspond to eigenvalues with larger magnitudes, since the noise accounts for more variability \citep{ding2020phase}. 
In such scenarios, it is advised to take larger $\gamma$ values, which can be motivated by the spectral decay. 
For example, as implied by \citet{ding2020phase}, the ratio between consecutive eigenvalues, i.e., $\lambda_i/\lambda_{i+1}$, can  distinguish between parts of the spectrum that relate to the data and parts that relate to the noise.
In our experiments, we set $\gamma\in\left[10^{-8},10^{-5}\right]$, which empirically led to good results.
\end{remark}

\subsubsection{Spectrum truncation and estimation\label{subsub:spect_est}}
The formulation in \eqref{eq:lemlb} employs the entire spectrum of the operator.
However, since we aim to compare datasets of different sizes, we truncate the spectrum and use only the leading (largest) eigenvalues.
This choice is motivated by the regularization above, which reduces the effects of smaller eigenvalues, and by the fact that smaller eigenvalues are commonly governed by noise.
Therefore, truncating the spectrum typically will not affect the ability to capture geometric structures in the data and is common in similar scenarios \citep{reuter2006laplace,cosmo2019isospectralization}.
\begin{remark}
Truncating the spectrum for comparing datasets of different sizes is significantly more advantageous than modifying the size of the dataset, e.g.\ by removing or repeating samples.
With direct data manipulations, geometric structures might be distorted or ignored, whereas with spectrum truncation, the dominant geometric structures are preserved, since they are typically captured by the largest eigenvalues. 
\end{remark}
\begin{remark}
Spectrum truncation preserves the lower bound of the LE metric in \eqref{eq:lemlb} since it removes nonnegative terms. 
\end{remark}

\textbf{Spectrum estimation.}
Computing even the truncated spectrum for large matrices is computationally demanding.
We use a recent spectral approximation technique based on the Nystr\"{o}m method for low-rank positive semidefinite (PSD) matrix approximation \citep{tropp2017fixed}.
Their approach implements a more stable truncation of the Nystr\"{o}m approximation to rank $K<N$ based on \citet{li2017algorithm}.

We outline the relevant parts of their method in Algorithm \ref{alg:evalest}.
This algorithm approximates the leading $K$ eigenvalues of a PSD matrix $\mathbf{W}$, with theoretical guarantees.
For the eigenvalue estimation in our setting, we obtain the following theoretical bounds based on a result by \citet{tropp2017fixed}: 
\begin{prop}\label{prop:evalest_bound}
The rank-$K$ approximation of a real symmetric PSD matrix, $\mathbf{W}$, computed according to Algorithm \ref{alg:evalest}, and denoted by $\hat{\mathbf{W}}_K$, satisfies the following error bounds.
\begin{enumerate}
\item The eigenvalue approximation error is bounded by:
\begin{flalign}
    &\mathbb{E}\left[\sum_{i=1}^K\left\vert\lambda_i^{(\mathbf{W})}\!-\!\lambda_i^{(\hat{\mathbf{W}}_K)}\right\vert\right]\leq\frac{K}{M\!-\!K\!-\!1}\sum_{i=K\!+\!1}^N\!\!\lambda_i^{(\mathbf{W})}\label{eq:eigbound}
\end{flalign}
\item The $\gamma$-regularized log-eigenvalue error is bounded by:
\begin{flalign}
    & \mathbb{E}\!\left[\sum_{i=1}^K\left\vert\log\!\left(\lambda_i^{(\mathbf{W})}\!+\!\gamma\right)\!-\!\log\!\left(\lambda_i^{(\hat{\mathbf{W}}_K)}\!+\!\gamma\right)\right\vert\right]
    \leq \frac{1.5K}{(M-K-1)(\lambda_K^{(\mathbf{W})}+\gamma)}\!\sum_{i=K\!+\! 1}^N\!\!\lambda_i^{(\mathbf{W})}\label{eq:logeigbound}
\end{flalign}
\end{enumerate}
for $\tfrac{|\lambda^{(\mathbf{W})}_i-\lambda^{(\hat{\mathbf{W}}_K)}_i|}{\lambda^{(\mathbf{W})}_i+\gamma}\!\leq\! 0.5828$, where the expectations are over the normal distribution (due to the random $\mathbf{\Omega}$ in Algorithm \ref{alg:evalest}), $\lambda_i^{(\mathbf{W})}$ and $\lambda_i^{(\hat{\mathbf{W}}_K)}$ denote the eigenvalues of $\mathbf{W}$ and $\hat{\mathbf{W}}_K$, respectively, organized in decreasing order, and $M$ denotes the number of random vectors (columns) in $\mathbf{\Omega}$.
\end{prop}
These bounds provide additional guidelines for setting $\gamma$ and $M$.
Note that due to the spectral decay, as $N$ increases, if we appropriately increase $M$, the error tends to $0$.

\begin{proof}
Applying Theorem 4.1 in \citet{tropp2017fixed} to \emph{symmetric} PSD matrices, we have
\begin{flalign}
    &\mathbb{E}\!\left[\left\Vert\mathbf{W}\!-\!\hat{\mathbf{W}}_K\right\Vert_1\right]
    \!\leq\!\left(\!1\!+\!\frac{K}{M\!-\!K\!-\!1}\right)\!\sum_{i=K\!+\!1}^N\!\!\lambda_i^{(\mathbf{W})},\label{eq:evalest_proof}
\end{flalign}
where $\left\Vert\cdot\right\Vert_1$ denotes the Schatten 1-norm.
Using Theorem III.4.4 in \citet{bhatia2013matrix} (see proof of Proposition \ref{prop:lelb}),
\begin{flalign}
    \left\Vert\mathbf{W}\!-\!\hat{\mathbf{W}}_K\right\Vert_1  = \sum_{i=1}^N\left\vert\lambda_i^{(\mathbf{W}-\hat{\mathbf{W}}_K)}\right\vert\geq \sum_{i=1}^N\left\vert\lambda_i^{(\mathbf{W})}\!-\!\lambda_i^{(\hat{\mathbf{W}}_K)}\right\vert=\sum_{i=1}^K\left\vert\lambda_i^{(\mathbf{W})}-\lambda_i^{(\hat{\mathbf{W}}_K)}\right\vert+\sum_{i=K+1}^N\lambda_i^{(\mathbf{W})},\nonumber
\end{flalign}
since $\mathbf{W}$ and $\hat{\mathbf{W}}_K$ are symmetric PSD and the $\ell^1$ norm is a symmetric gauge function. 
The last equality is due to the construction of $\hat{\mathbf{W}}_K$ with rank $K$. 
Substituting this result into \eqref{eq:evalest_proof} and rearranging terms yields \eqref{eq:eigbound}.
The proof of the second part of the proposition appears in Appendix \ref{app:prop2proof}.
\end{proof}
The proposition holds in our case, since the matrices for which we estimate the eigenvalue decomposition, i.e., $\mathbf{W}$ in \eqref{eq:dm_spd}, are symmetric positive-definite (SPD).

Proposition \ref{prop:evalest_bound} demonstrates that the eigenvalue estimation error depends on the spectrum tail, the spectral decay, and the number $M$ of random vectors.
The spectral decay of our $\mathbf{W}$ can be controlled to some extent by the kernel scale, $\sigma^2$ in \eqref{eq:affkern}, where larger $\sigma^2$ yields faster spectral decay.
In our experiments, we set $M=2K$ to obtain a reasonable coefficient multiplying the spectrum tail in the bound.

The computational complexity of this approximation algorithm is $\mathcal{O}(MN^2+M^2N)$ \citep{tropp2017fixed}. For very large datasets, we have $M\ll N$ due to  spectrum truncation.

Note that the dense matrix $\mathbf{W}$ in \eqref{eq:dm_spd} requires storing $N^2$ floating point numbers, which may be limiting in large datasets.
To reduce this cost, the matrix $\mathbf{W}$ can be computed in small batches of rows (in two iterations due to its normalization), since Algorithm \ref{alg:evalest} relies on $\mathbf{W}$ only through matrix multiplications with $\mathbf{\Omega}$.

\setlength{\textfloatsep}{5pt}
\begin{algorithm}[tb]
  \caption{Fixed-Rank approx.\ \citep{tropp2017fixed}}
  \label{alg:evalest}
%   \footnotesize
\begin{algorithmic}
  \REQUIRE PSD matrix $\mathbf{W}\in\mathbb{R}^{N\times N}$, rank $K$, size $M$. 
  \ENSURE Diagonal matrix $\mathbf{\Lambda}\in\mathbb{R}^{K\times K}$, such that 
  $\mathbf W\approx\hat{\mathbf{W}}_K=\mathbf{U}\mathbf{\Lambda}\mathbf{U}^T$, for some $\mathbf{U}\in\mathbb{R}^{N\times K}$ with orthonormal columns.
  \STATE
  \STATE {\bfseries Function} \textsc{ApproxEigenvalues}$\left(\mathbf{W},K, M\right)$
  \begin{ALC@g}
    \STATE  $\mathbf{\Omega} \leftarrow \texttt{randn}\left(N,M\right)$
    \STATE $\mathbf{\Omega} \leftarrow \texttt{orth}\left(\mathbf{\Omega}\right)$ \COMMENT{Orthonormal basis for $\mathrm{range}(\mathbf{\Omega})$}
    \STATE $\mathbf{Y}\leftarrow \mathbf{W}\mathbf{\Omega}$
    \STATE $\nu\leftarrow \mu\ \texttt{norm}(\mathbf{Y})$ \COMMENT{$\mu=2.2\cdot 10^{-16}$}
    \STATE $\mathbf{Y}\leftarrow\mathbf{Y}+\nu\mathbf{\Omega}$
    \STATE $\mathbf{B}\leftarrow \mathbf{\Omega^T}\mathbf{Y}$
    \STATE $\mathbf{C}\leftarrow \texttt{chol}\left((\mathbf{B}+\mathbf{B^T})/2\right)$ \COMMENT{Cholesky decomp.}
    \STATE $\left(\mathbf{U},\mathbf{\Sigma},\sim\right)\leftarrow\texttt{svd}\left(\mathbf{Y}/\mathbf{C}\right)$
    \STATE $\mathbf{\Lambda}\leftarrow\max\{0,\mathbf{\Sigma}^2_{1:K,1:K}-\nu\mathrm{\mathbf{I}}\}$
    \STATE {\bfseries return} sorted $\mathbf{\Lambda}$
  \end{ALC@g}
\end{algorithmic}
\end{algorithm}

\subsection{The Log-Euclidean Signature Distance\label{sub:lesd}}
We can now define the log-Euclidean signature distance between unaligned datasets. 
Given two datasets $\{x^{(1)}_i\in\mathbb{R}^{d_1}\}_{i=1}^{N_1}$ and $\{x^{(2)}_i\in\mathbb{R}^{d_2}\}_{i=1}^{N_2}$, we construct diffusion operators  $\mathbf{W}_1\in\mathbb{R}^{N_1\times N_1}$ and $\mathbf{W}_2\in\mathbb{R}^{N_2\times N_2}$ according to \eqref{eq:dm_spd}. 
We approximate the $K$ leading eigenvalues $\{\hat{\lambda}^{(\ell)}_i\}_{i=1}^K$ for each matrix using Algorithm \ref{alg:evalest}. 
The LES distance is then:
\begin{flalign}
    \boxed{d^2_{\mathrm{LES}}\!\left(\mathbf{W}_1,\mathbf{W}_2\right)\!:=\!
    \sum_{i=1}^K\!\left(\log\!\left(\hat{\lambda}^{(1)}_i\!+\!\gamma\right)\!-\!\log\!\left(\hat{\lambda}^{(2)}_i\!+\!\gamma\right)\!\right)^2}
    \label{eq:lesdist}
\end{flalign}
This procedure is summarized in Algorithm \ref{alg:lesd}.

\begin{remark}
The LES distance is not a metric, since it does not satisfy the identity of indiscernibles, i.e. $d_{LES}(\mathbf{W}_1,\mathbf{W}_2)\!\neq\! 0$ is not guaranteed for $\mathbf{W}_1\!\neq\!\mathbf{W}_2$. Specifically, datasets with underlying non-isometric but isospectral manifolds (or isospectral up to the $K$th eigenvalue) are indistinguishable.
Isospectral manifolds exist \citep{kac1966can} but typically require careful construction and are rarely encountered in practice, especially in low dimensions \citep{reuter2006laplace}.
In addition, for $\mathbf{W}_1\!\equiv\!\mathbf{W}_2$,  $d_{LES}(\mathbf{W}_1,\mathbf{W}_2)\!=\!0$ is guaranteed only with exact eigenvalue computation. However, our experimental results depict that the distance is close to $0$ also with the eigenvalue approximation, and even for \emph{different} point clouds sampled from the same manifold.
\end{remark}
\begin{remark}
The proposed framework is not limited to diffusion operators $\mathbf{W}$. 
While our choice of $\mathbf{W}$ is motivated by its convergence to the heat kernel, the LES distance can be computed using any SPD matrix, e.g., covariance matrices.
\end{remark}

\textbf{Relation to IMD.} The LES distance is directly related to IMD \citep{tsitsulin2019shape}, written using our notation as 
$$
d_{\mathrm{IMD}}(\mathbf{W}_1,\mathbf{W}_2)=\sup_{t>0}c(t)\left\vert \sum_i(\lambda_i^{(1)})^t-\sum_i(\lambda_i^{(2)})^t\right\vert
$$
where $c(t)=e^{-2(t+t^{-1})}$, $\lambda_i^{(\ell)}=\exp(-\mu_i^{(\ell)})$ and $\mu_i^{(\ell)}$ are (approx.) Laplace-Beltrami eigenvalues.
This formulation demonstrates two key differences between LES and IMD.
First, LES compares aligned spectral locations, whereas in IMD the overall decay of the spectrum is compared in different temperature scales, controlled by the variable $t$.
Second, $\log$ is applied to the eigenvalues by LES but not IMD, indicating that by taking a proper metric into account when comparing discrete approximations of heat kernels, LES actually compares approximate Laplace-Beltrami eigenvalues (related to eigenvalues of the heat kernel via a $\log$), while IMD compares heat kernel eigenvalues, which decay quickly.
We demonstrate empirically in Section \ref{sec:results} that the LES distance outperforms IMD in all considered applications, particularly when the datasets contain complex structures with differences in multiple scales.

\begin{algorithm}[hb]
  \caption{Log-Euclidean Signature Distance}
  \label{alg:lesd}
%   \footnotesize
\begin{algorithmic}
  \REQUIRE $r$ datasets $X^{(\ell)}=\left\lbrace x^{(\ell)}_i\in\mathbb{R}^{d_\ell}\right\rbrace_{i=1}^{N_\ell}$, $\ell=1,\dots,r$.
  \ENSURE LES descriptors, $F_{LES}^{(\ell)}\in\mathbb{R}^K$, and LES distances, $d_{LES}(\ell,q)$, for $\ell, q\in\left[1,r\right]$.
  \STATE
  \FOR{$\ell=1$ {\bfseries to} $r$}
    \STATE $\mathbf{W}_{\ell}\leftarrow \texttt{DM-SPD}\left(X^{(\ell)}\right)$ \COMMENT{According to \eqref{eq:dm_spd}}
    \STATE $\left\lbrace\hat{\lambda}_i^{(\ell)}\right\rbrace_{i=1}^K\leftarrow \textsc{ApproxEigenvalues}(\mathbf{W}_{\ell})$
    \STATE $F_{\mathrm{LES}}^{(\ell)}\leftarrow \left[\log\left(\hat{\lambda}_1^{(\ell)}+\gamma\right),\dots,\log\left(\hat{\lambda}_K^{(\ell)}+\gamma\right)\right]$
  \ENDFOR
  \FOR{$\ell,q\in\left[1,r\right]$}
    \STATE $d^2_{\mathrm{LES}}(\ell,q)\leftarrow \left\Vert F_{LES}^{(\ell)}-F_{LES}^{(q)} \right\Vert_2^2$
  \ENDFOR
\end{algorithmic}

\end{algorithm}

\textbf{Spectral descriptors of data.}
Spectral descriptors appear in several applications, e.g., in shape analysis \citep{reuter2006laplace,cosmo2019isospectralization} and molecular representations \citep{schrier2020can}.
We similarly define a dataset descriptor based on the bound in \eqref{eq:lemlb} and the proposed distance \eqref{eq:lesdist}, as $$F_{\mathrm{LES}}^{(\ell)}= \left[\log\left(\hat{\lambda}_1^{(\ell)}+\gamma\right),\dots,\log\left(\hat{\lambda}_K^{(\ell)}+\gamma\right)\right].$$ 
$F_{\mathrm{LES}}^{(\ell)}$ embeds datasets into a $K$-dimensional space based on intrinsic geometry, facilitating dimensionality reduction and visualization and allowing for intuitive analysis of clusters.

\section{Experimental Results\label{sec:results}}
Throughout this section we compute LES according to Algorithm \ref{alg:lesd} with $\sigma^2=2\,\mathrm{median}\left(d^2(x_i,x_j)\right)$, $K\in\lbrace 200,500\rbrace$, $M=2K$ and $\gamma\in[10^{-8},10^{-5}]$.
The exact parameters used in each application are described in the appendices.
An additional experiment using LES to match hidden representations of NN is presented in Appendix \ref{app:layers}.

The code is available in \url{https://github.com/shnitzer/les-distance}.

\subsection{Toy Example: Comparison of Geometric Shapes\label{sub:exp_tori}}
We demonstrate the main advantages of the LES distance over competing methods on a toy example consisting of $3$D and $4$D point clouds.
We sample points from two $2$D tori, $T_{2}=\mathcal{S}^1\times\mathcal{S}^1$, and two $3$D tori, $T_{3}=\mathcal{S}^1\times\mathcal{S}^1\times\mathcal{S}^1$, which have the same major radius $R_1$.
The two $T_{2}$ tori differ in their minor radius, denoted by $R_2$ and $R_2^{Sc}=c\!\cdot\! R_2$, and the two $T_{3}$ tori share $R_2$ and differ in their minor radius, denoted by $R_3$ and $R_3^{Sc}=c\!\cdot\! R_3$.
The torus equations are formulated such that $\lim_{R_3\rightarrow 0}T_{3}=T_{2}$, and $\lim_{R_2\rightarrow 0}T_{2}=T_{1}$, i.e., converges to a $1$D circle.
We denote the two tori with the scaled radius by $T_2^{Sc}$ and $T_3^{Sc}$. 
We sample the $4$ tori independently such that no pointwise correspondence exists.
The exact equations of the tori embedding are given in Appendix \ref{app:tori_exp}, eq. \eqref{appeq:tori_eqs}, along with a detailed description of the chosen parameters in this subsection.

We compute distances between the four tori for different $c$'s using LES, IMD, Geometry Score (GS) \citep{khrulkov2018geometry}, and Gromov-Wasserstein (GW) \citep{peyre2016gromov}.
To emphasize the conceptual differences between LES and IMD, we approximate IMD according to $d_{\mathrm{IMD}}(\mathbf{W}_1,\mathbf{W}_2)$ (Section \ref{sub:lesd}), using our estimated eigenvalues of $\mathbf{W}$. 
We denote this method by \emph{IMD (our approach)}.
Due to the different dimensions of the tori datasets, methods requiring that all datasets are embedded in the same space, such as MTD and CrossLID, cannot be applied. We compare our results with other standard topological persistence diagram metrics in Appendix \ref{appsub:tori_topo}.

\setlength{\textfloatsep}{\textfloatsepsave}
\begin{figure}[ht]
\centering
\subfigure{
\includegraphics[width=0.7\textwidth]{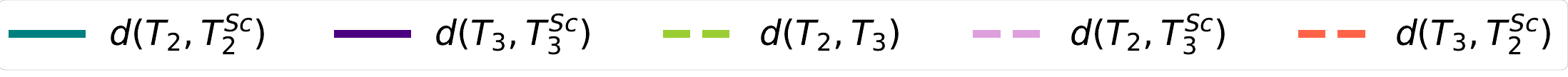}
}\vspace{-0.2in}

\setcounter{subfigure}{0}
\subfigure[LES]{
    \includegraphics[trim={0 0.2in 0 0}, clip, width=0.2\textwidth]{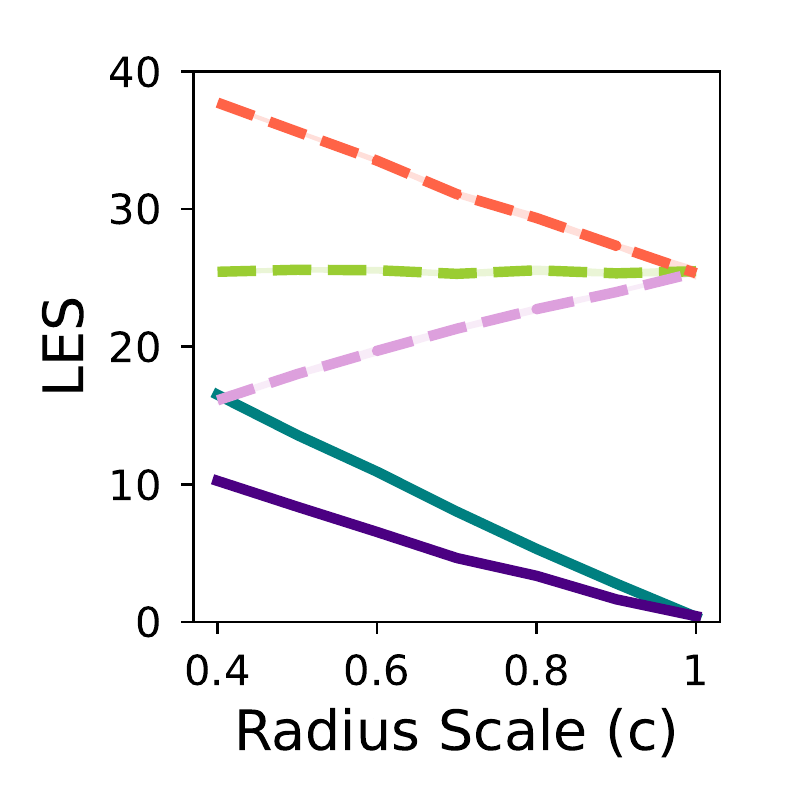}
    \label{fig:tori_les}
}
\hspace{-0.3in}
\subfigure[IMD]{
    \includegraphics[trim={0 0.2in 0 0}, clip, width=0.2\textwidth]{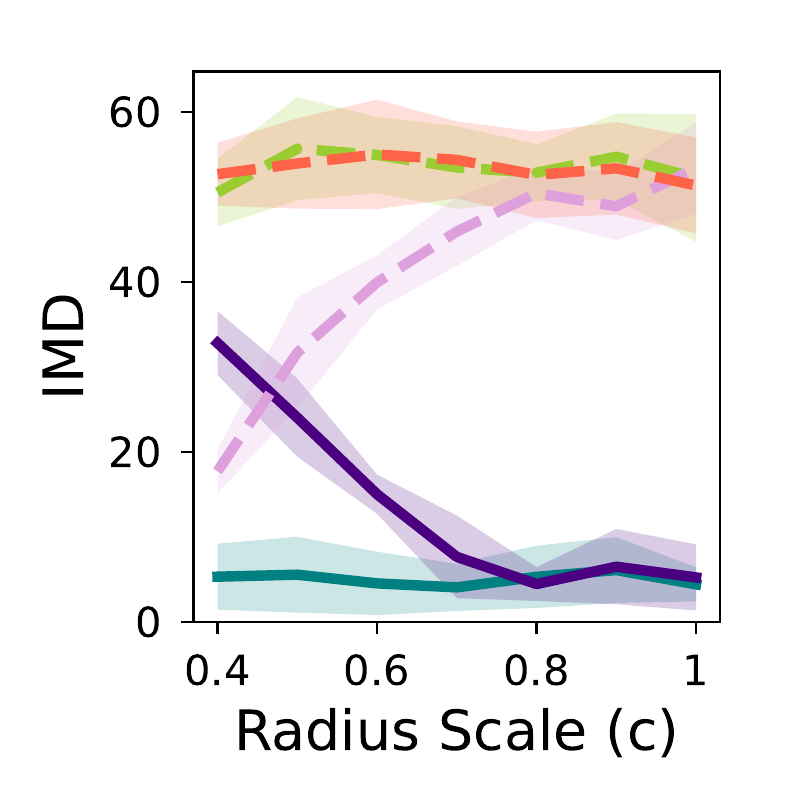}
    \label{fig:tori_imd}
}
\hspace{-0.3in}
\subfigure[IMD (our approach)]{
    \includegraphics[trim={0 0.2in 0 0}, clip, width=0.2\textwidth]{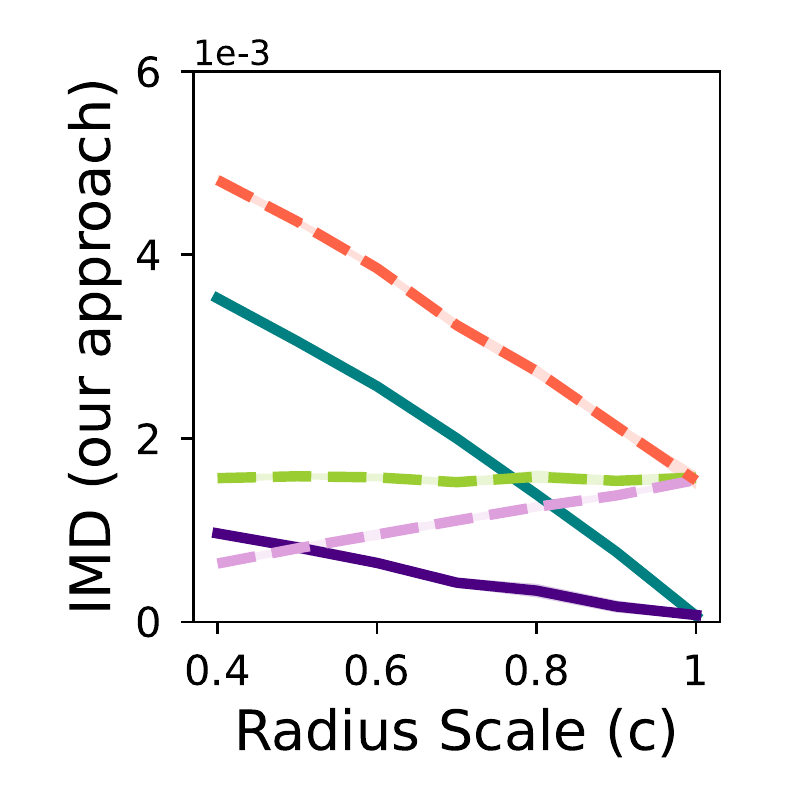}
    \label{fig:tori_imdours}
}
\hspace{-0.3in}
\subfigure[Geometry Score]{
    \includegraphics[trim={0 0.2in 0 0}, clip, width=0.2\textwidth]{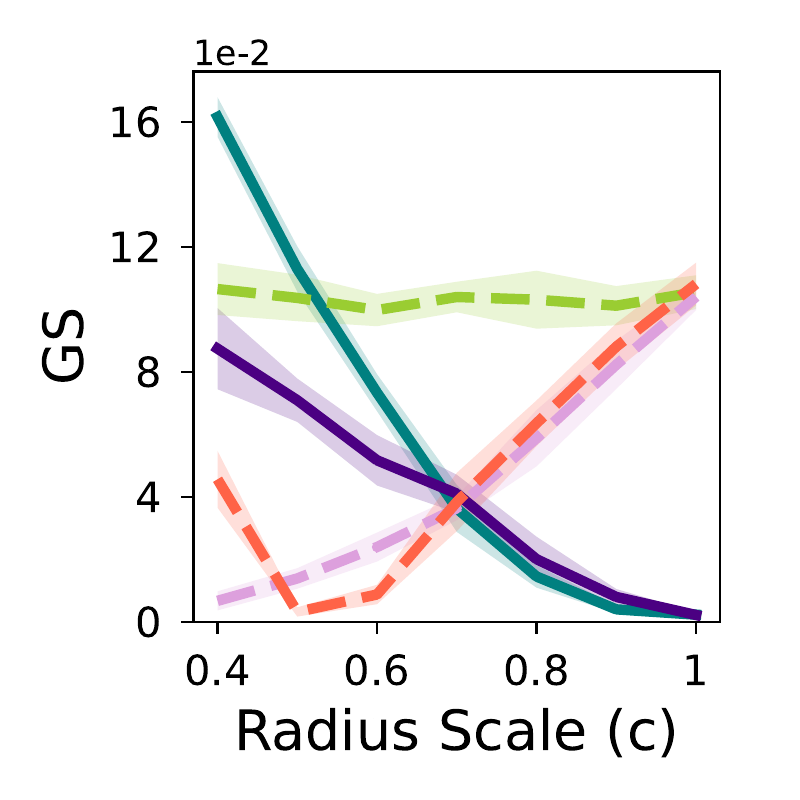}
    \label{fig:tori_gs}
}
\hspace{-0.3in}
\subfigure[Gromov-Wasserstein]{
    \includegraphics[trim={0 0.2in 0 0}, clip, width=0.2\textwidth]{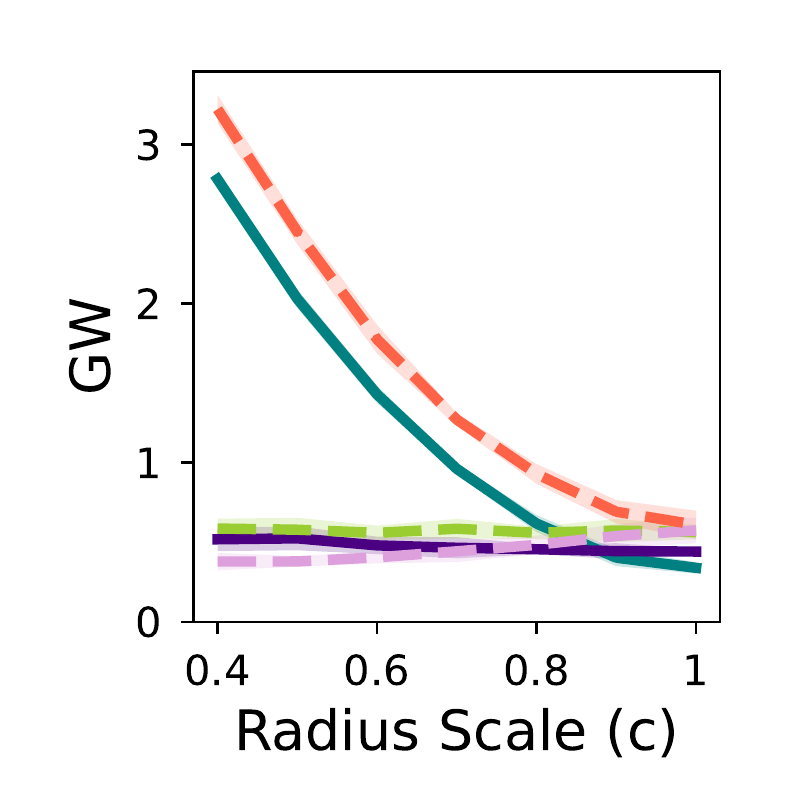}
    \label{fig:tori_gw}
}

\caption{Comparison of distance measures on $2$D and $3$D tori with radii related by scaling $c$.\label{fig:tori_dist}}
\vskip -0.1in 
\end{figure}

\Cref{fig:tori_dist} presents the mean and standard deviation of the distances between the different tori, averaged over $10$ trials.
Solid lines denote distances between tori with the same dimensionality and dashed lines denote distances between tori with different dimensions.
As $c$ decreases, it is expected that distances $d(T_2,T_2^{Sc})$, $d(T_3,T_3^{Sc})$ and $d(T_3,T_2^{Sc})$ will increase, since they converge to a distance between shapes of different dimensions, and that $d(T_2,T_3^{Sc})$ will decrease, since $\lim_{c\rightarrow 0}T_{3}^{Sc}=T_{2}$.
Based on \cref{fig:tori_dist}, the only methods that exhibit the expected trends as $c$ decreases are LES, GW, and IMD computed using our approach, which indicates that our approach for estimating leading heat kernel eigenvalues yields a more stable and accurate distance.
The IMD distance in \cref{fig:tori_imd} captures mainly the major dimensionality differences and is not sensitive to small differences in the radius scale.
Moreover, the LES distance distinguishes between dimensionality differences and scale differences for a wider range of $c$ values compared with GW and IMD (our approach), as depicted by the intersection between the distance lines of $d(T_2,T_2^{Sc})$ and $d(T_2,T_3^{Sc})$.

The distance trends captured by LES follow the trends of the exact log-Euclidean distances between the SPD matrices representing the tori data with ground truth pointwise correspondence. We present an example in Appendix \ref{appsub:tori_addres}.

\textbf{Running time and stability.}
Using the tori simulation, we compare computational complexity and stability of LES with IMD (with the exact $k$-nn distance), GW, and GS, for different dataset sizes, $N$.
\Cref{fig:tori_time} demonstrates that LES and IMD require significantly less computational time than GW and GS.
While LES requires more computational time than IMD, its advantages in stability and accuracy are significant, as depicted in the examples in this section.

To evaluate the effect of different sample sizes on LES, we compute the distance $d(T_2,T_3)$ for different $N$ values.
\Cref{fig:tori_stab} presents this distance normalized by the distance at $N=10^4$. 
This plot demonstrates that LES is the most robust to different sample sizes, even for relatively small $N$.

\begin{figure}[ht]
\centering
\subfigure{
\includegraphics[width=0.35\textwidth]{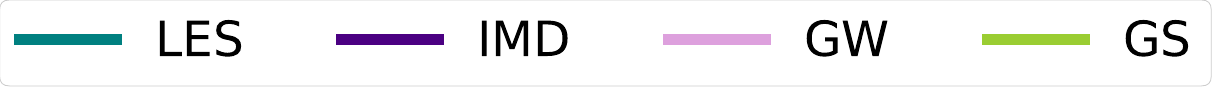}
}\vspace{-0.15in}

\setcounter{subfigure}{0}
\subfigure[Run time]{
    \includegraphics[trim={0 0.2in 0 0}, clip, width=0.22\textwidth]{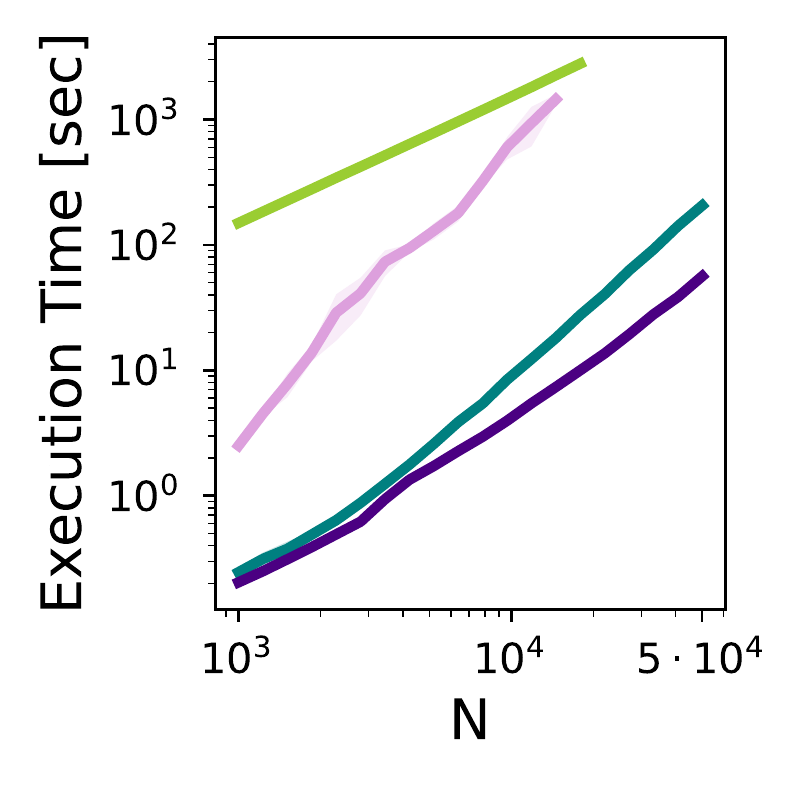}
    \label{fig:tori_time}
}
\hspace{-0.2in}
\subfigure[Normalized Distance]{
    \includegraphics[trim={0 0.2in 0 0}, clip, width=0.22\textwidth]{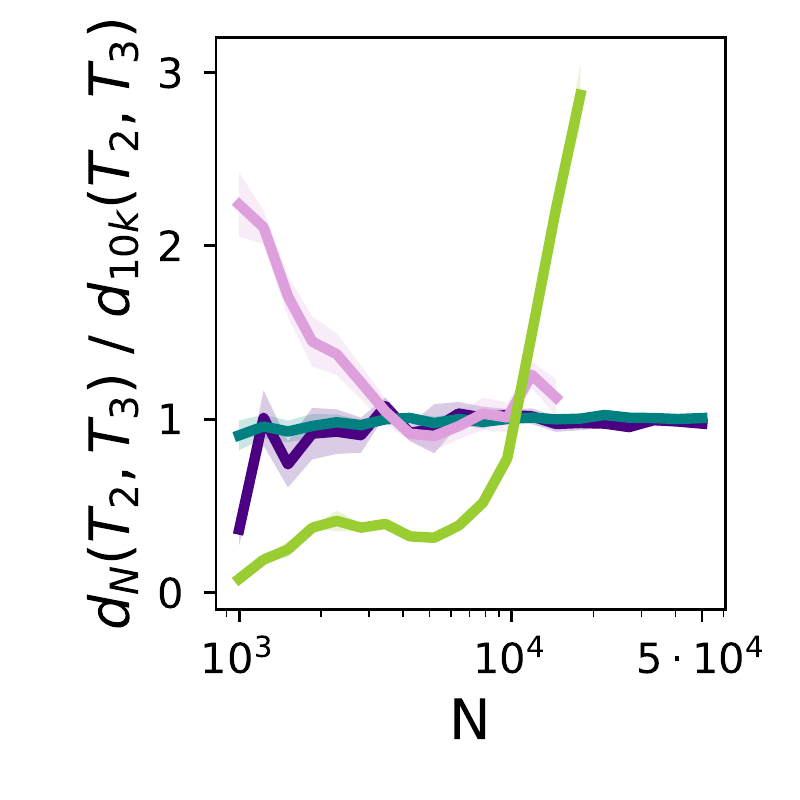}
    \label{fig:tori_stab}
}
\vskip -0.15in 
\caption{Run time and stability for different dataset sizes.\label{fig:tori_time_stab}\vspace{-.2in}}
\end{figure}

\subsection{Single Cell Gene Expression Data Analysis\label{sub:exp_genes}}
We further illustrate the power of LES in recovering fine structural details on single-cell RNA sequencing (scRNA-seq) data.
ScRNA-seq is powerful for analyzing cell populations and developmental trajectories \citep{tanay2017scaling}.
Recent works propose visualization techniques of such data for the analysis/comparison of cell sub-populations \citep{haghverdi2015diffusion, moon2019visualizing}.
However, comparing datasets of scRNA-seq in different gene spaces, e.g., from different species or measurement modalities, remains a challenge \citep{shafer2019cross}.
Since in such tasks a clear ground truth is typically unavailable, we demonstrate the applicability of LES to identifying intrinsic structural similarities between gene-expression datasets via an illustrative task of recovering the time axis in cell differentiation data. 
In such experiments, all the cells start from a stem cell-like state and differentiate into various cell types. 
The gene-expressions of (random, non-repeating) subsets of these cells are measured at different time points throughout the developmental trajectory.
Therefore, the time axis relates to the developmental stage of the cells and can indicate the expected similarity of the different subsets.
% Using an illustrative task of recovering the time axis in cell differentiation data, we demonstrate the applicability of LES to identifying intrinsic structural similarities between such datasets.
This application inherently lacks pointwise alignment between samples, since scRNA-seq destroys the cell during measurement.

We apply LES to data of reprogrammed mouse embryonic cells \citep{schiebinger2019optimal}, used for learning developmental trajectories.
It contains expressions of $\sim\!2.5\!\cdot\! 10^5$ cells and $1500$ genes (after filtration), collected across $18$ days of differentiation in half day intervals, resulting in $39$ time points; a visualization is in Appendix \ref{app:gene_exp}.
Each time point consists of gene expressions of $2000\mbox{--}12000$ cells from the same developmental stage. 
Our goal is to measure the gene-space similarities of these different cell groups.
Each cell group (from every time point) is treated as a separate dataset, and we measure the distances between every pair of datasets using LES, IMD, Wasserstein distance (OT), GS, and MTD \citep{barannikov2021manifold}; we omit GW since it requires unreasonable processing times for all pairs of $39$ time points.
To visualize the resulting distance matrices, we compute their embeddings using diffusion maps \citep{coifman2006diffusion}. Figure \ref{fig:gene_dist} presents the leading embedded coordinates, denoted by $f_1$ and $f_2$, where each point corresponds to a cell group from one time. The points are colored according to their measurement day, denoting their location along the developmental trajectory.
% At each time point, we represent the population in gene-space using the LES descriptors, $F_{\mathrm{LES}}$; we measure the distances between populations at different times.
% Figure \ref{fig:gene_dist} compares LES to IMD and GS; we omit GW since it requires unreasonable processing times for all pairs of $39$ time points. 
% Figure \ref{fig:gene_dist} presents diffusion maps embeddings \citep{coifman2006diffusion} of the distances, where $f_1$ and $f_2$ denote the leading embedded coordinates (eigenvectors).
% Each point represents one time and is colored according to its measurement day, denoting its location along the developmental trajectory. 
The correlations between the time axis and the distances are below each plot.
LES captures the developmental trajectory significantly better than IMD, GS, and MTD, as emphasized by their correlation with time and plot coloring. 
LES obtains a slightly higher correlation with time than OT, even though OT uses geometric proximity of the datasets while LES is intrinsic.

To slightly complicate the problem, we repeat this experiment and sample a different subset of genes at each time point ($\sim \mathrm{U}[1200,1500]$), so that cells from different days are represented in gene spaces of different dimensions.
LES outperforms IMD in this setting as well, with slightly reduced correlations of $0.69$ (LES) and $0.27$ (IMD).
Computation of OT and MTD in this setting is not possible, since cell groups from different time points are in different domains (different gene subsets and dimensionality).
These results suggest that LES could be highly useful in comparing cell population structures from different gene spaces.

\begin{figure*}[ht]
\centering
\subfigure{
\includegraphics[width=0.3\textwidth]{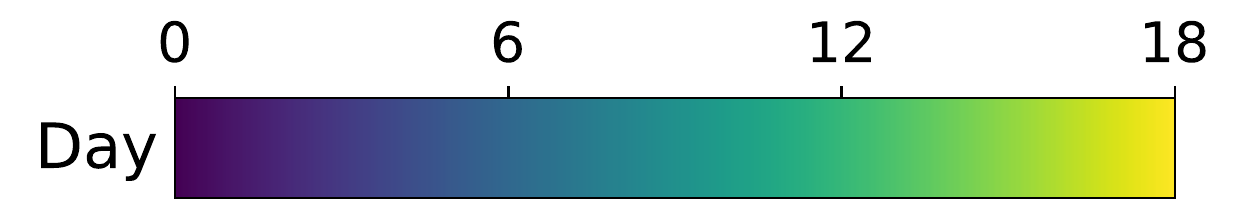}
} \vspace{-0.1in}

\setcounter{subfigure}{0}
\subfigure[LES: $0.74$]{
    \includegraphics[trim={0 0.2in 0 0}, clip, width=0.16\textwidth]{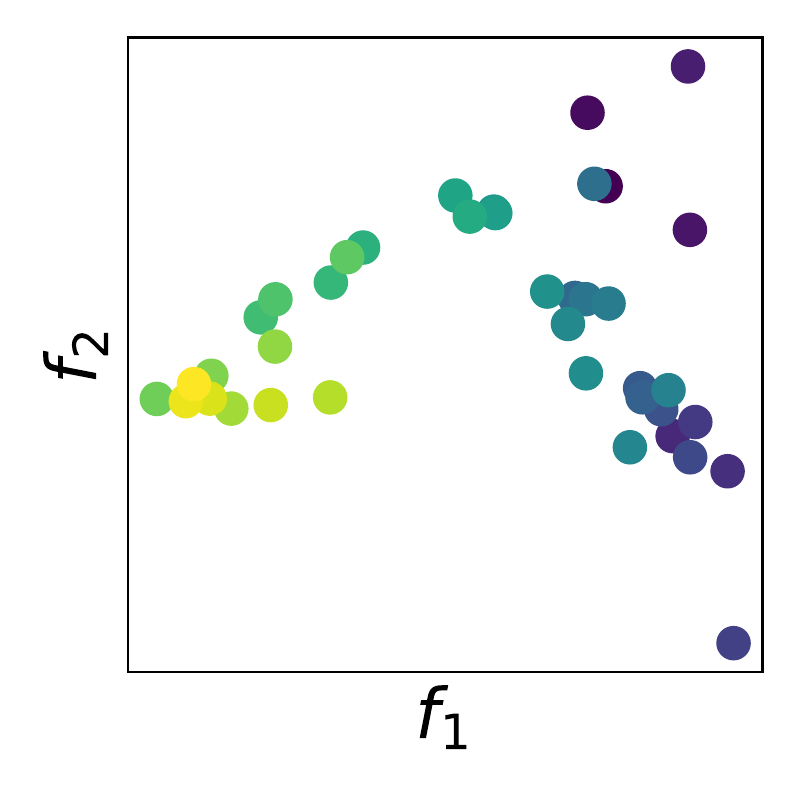}
    \label{fig:gene_les}
}
\subfigure[IMD: $0.28$]{
    \includegraphics[trim={0 0.2in 0 0}, clip, width=0.16\textwidth]{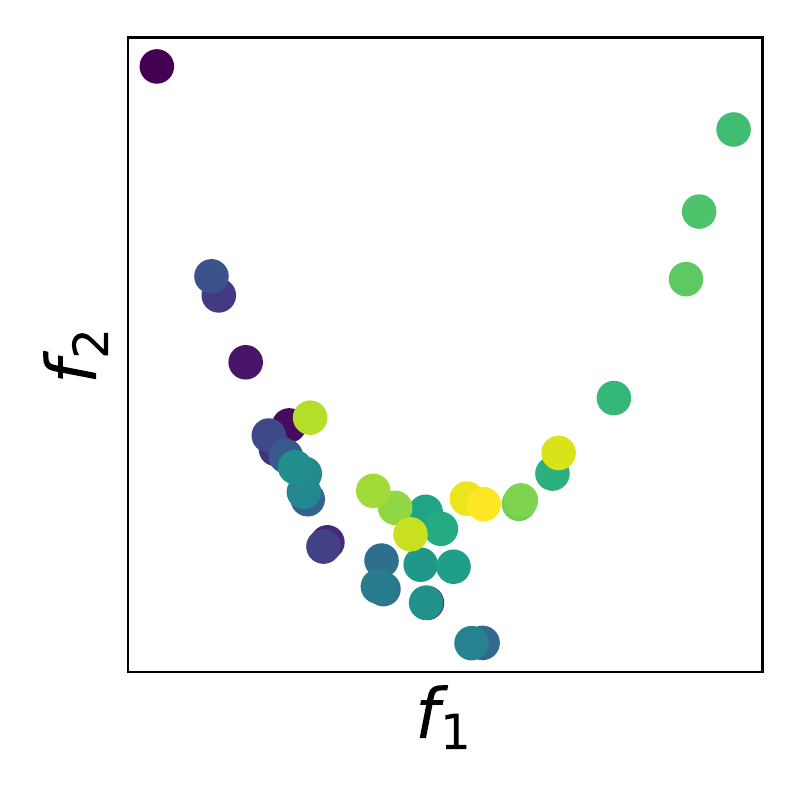}
    \label{fig:gene_imd}
}
\subfigure[OT: $0.6$]{
    \includegraphics[trim={0 0.2in 0 0}, clip, width=0.16\textwidth]{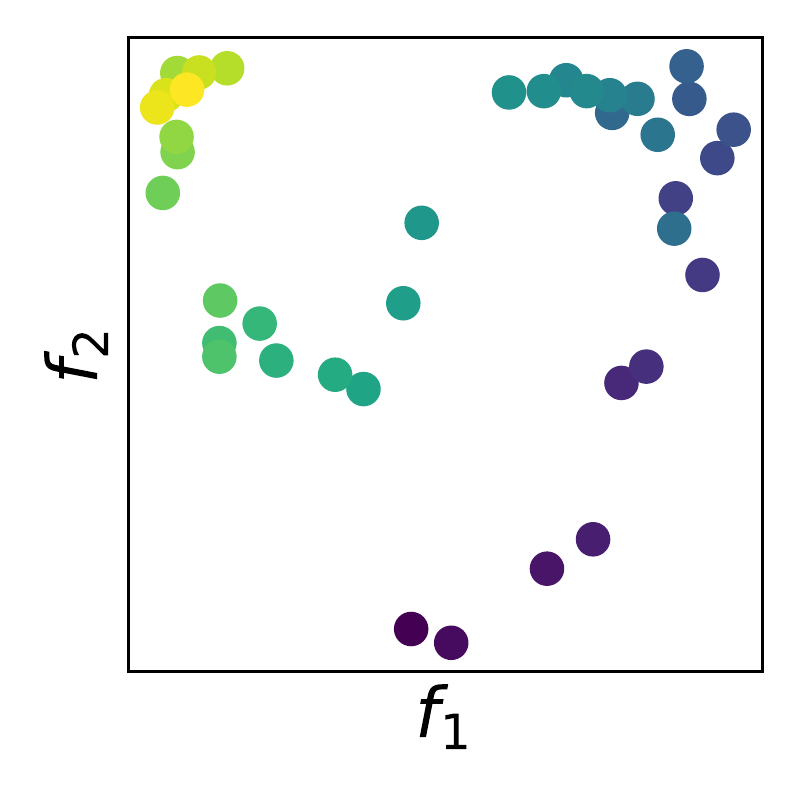}
    \label{fig:gene_ot}
}
\subfigure[GS: $0.03$]{
    \includegraphics[trim={0 0.2in 0 0}, clip, width=0.16\textwidth]{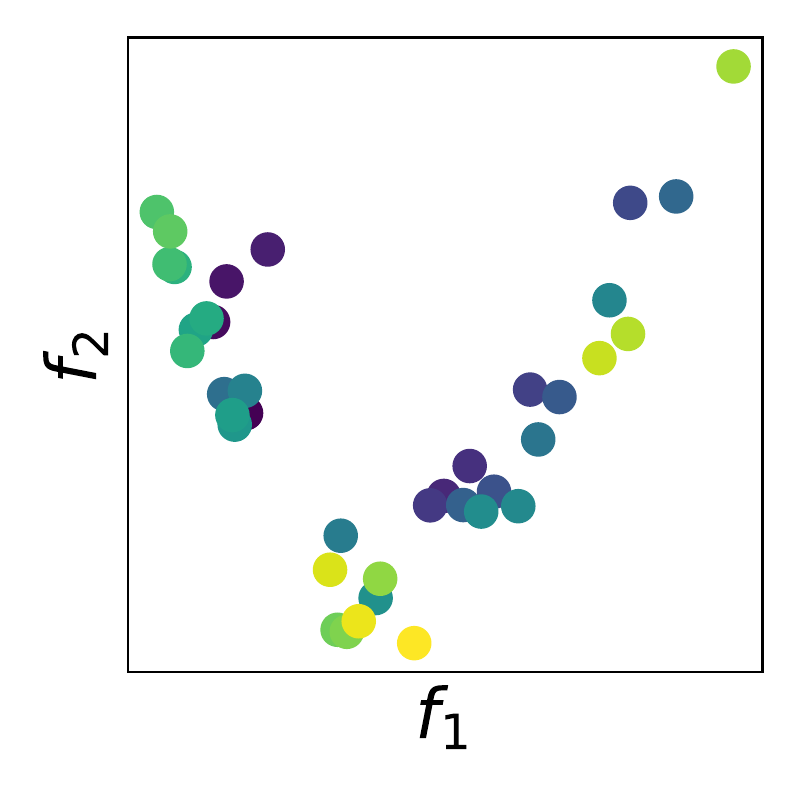}
    \label{fig:gene_gs}
}
\subfigure[MTD: $0.1$]{
    \includegraphics[trim={0 0.2in 0 0}, clip, width=0.16\textwidth]{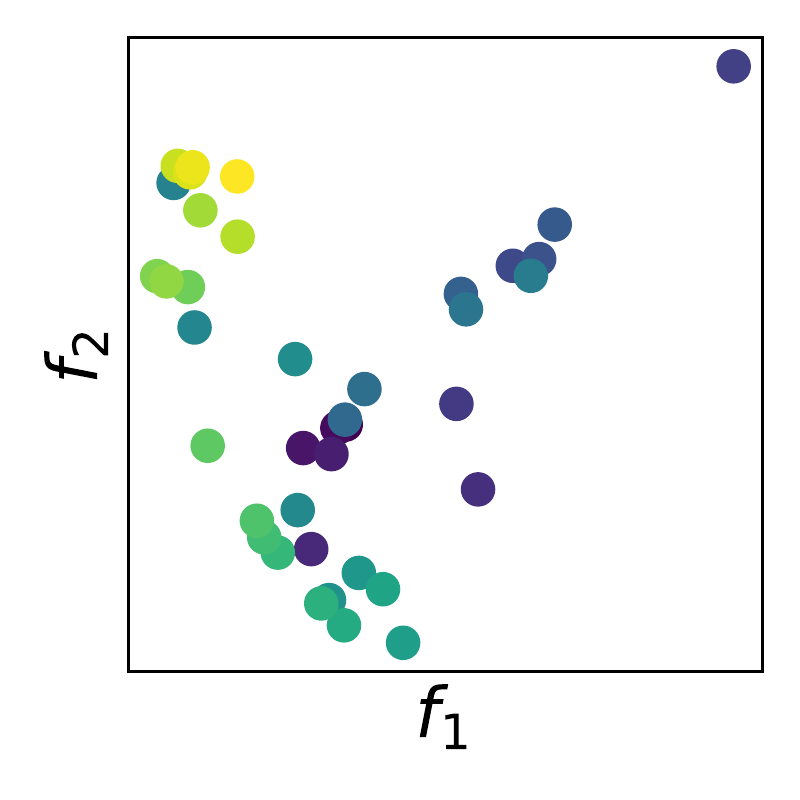}
    \label{fig:gene_mtd}
}
% \vskip -0.1in
\caption{Gene-expression datasets embedding based on the different distances. Correlation with time is reported by the method labels.\label{fig:gene_dist}\vspace{-.1in}}
% \vskip -0.1in
\end{figure*}

\begin{table}[t]
% \vskip -0.05in
\caption{Few-shot learning: MetaOptNet test accuracies and correlations (20 trials) with task dissimilarities.}
\label{table:meta-mean}
\begin{center}
\scalebox{.85}{
\begin{tabular}{lcccccc}
\toprule
{} &               FC100 1-shot &               FC100 5-shot &              FC100 10-shot &            CIFAR-FS 1-shot &            CIFAR-FS 5-shot &           CIFAR-FS 10-shot \\
\midrule
Acc. & 38.19$\pm$0.48\%   &    54.45$\pm$0.49\%     &    60.52$\pm$0.48\%     &    70.79$\pm$0.69\%  &   83.98$\pm$0.44\%  &   87.11$\pm$0.40\% \\
\midrule
LES &  \textbf{-0.934}$\pm$0.034 &  \textbf{-0.945}$\pm$0.032 &  \textbf{-0.935}$\pm$0.032 &           -0.671$\pm$0.128 &  \textbf{-0.698}$\pm$0.155 &  \textbf{-0.657}$\pm$0.175 \\
IMD &           -0.184$\pm$0.250 &           -0.303$\pm$0.225 &           -0.210$\pm$0.235 &           -0.151$\pm$0.225 &           -0.015$\pm$0.220 &           -0.032$\pm$0.281 \\
OT  &            0.739$\pm$0.102 &            0.605$\pm$0.126 &            0.579$\pm$0.133 &            0.453$\pm$0.170 &            0.269$\pm$0.214 &            0.180$\pm$0.231 \\
GW  &           -0.919$\pm$0.034 &           -0.921$\pm$0.046 &           -0.914$\pm$0.045 &  \textbf{-0.677}$\pm$0.126 &           -0.672$\pm$0.167 &           -0.582$\pm$0.171 \\
GS  &           -0.735$\pm$0.101 &           -0.570$\pm$0.154 &           -0.581$\pm$0.218 &           -0.384$\pm$0.201 &           -0.034$\pm$0.258 &           -0.012$\pm$0.281 \\
\bottomrule
\end{tabular}
}
\end{center}
\vspace{-.1in} 
\end{table}

\subsection{Identifying Difficult Tasks in Few-Shot Learning}
\label{sec:meta-learning}

Meta-learning algorithms, e.g., Model Agnostic Meta Learning \citep{finn2017model} and MetaOptNet \citep{lee2019meta}, enable training models that can adapt to new tasks. In few-shot learning, they obtain classifiers for previously unseen classes using as little as a single labeled example per class. However, meta-learning can fail on tasks that are not sufficiently similar to the training tasks \citep{fallah2021generalization}.
% \my{try to find some empirical paper}.
To anticipate such failures a practitioner may consider measuring dissimilarity between the tasks used during training and the incoming test task. A good dissimilarity score correlates negatively with the performance on tasks to provide guidance on whether it is safe to deploy a meta-learning model on a new task without collecting labels. Namely, in such a setting we aim to predict whether a specific embedder (NN) will result in high few-shot classification accuracy in an unsupervised manner. 
Our study complements recent literature on meta-learning safety \citep{goldblum2020adversarially,slack2020fairness,agarwal2021sensitivity}.

Suppose we have several training tasks (datasets) $\{\widetilde{X}^{(\ell)} \!\in\! \mathbb{R}^{N_\ell \times d}\}_{\ell=1}^n$ and a new task $\widetilde{X}' \!\in\! \mathbb{R}^{N' \times d}$.
% We also compute accuracy on $X'$ after adapting the meta-learner on a corresponding number of labeled samples (shots) to evaluate the correlation between the dissimilarity and the accuracy.
We train MetaOptNet-SVM \citep{lee2019meta} for benchmark datasets CIFAR-FS \citep{bertinetto2018meta} and FC100 \citep{oreshkin2018tadam} in 1-, 5-, and 10-shot settings. Its architecture has a feature extractor shared across tasks, generating task embeddings denoted by $X^{(\ell)}$ and $X'$ for train and new tasks, respectively, and a classification head for each new task, learned during adaptation for evaluating accuracy.
We compute dissimilarity between embeddings of the new and train tasks as $\mathcal{A}(\{d(X^{(\ell)}, X')\}_{\ell=1}^n)$, where $\mathcal{A}$ is some aggregator, e.g.\ the mean, and $d(\cdot,\cdot)$ is a dataset distance.

For each benchmark and shot setting we sample several new tasks from train/test task distributions and their mixtures. 
In Table \ref{table:meta-mean} we report the correlation between the task accuracy and the dissimilarity score (with $\mathcal{A}$ as the mean) based on various dataset distances. We compare LES against IMD, GW, GS and Wasserstein distance (OT) (here the dimensionality of all datasets is the same, so OT can be applied). LES outperforms IMD, GS and OT in all settings by a large margin. Performance of GW is comparable, but it is orders of magnitude slower. The accuracy on the test tasks is also reported in Table \ref{table:meta-mean}, which clearly distinguishes the two datasets (accuracy on train tasks for both is above 90\% in all settings). FC100 is a more challenging variant of CIFAR-FS that was created to reduce similarity between train and test tasks \citep{oreshkin2018tadam}. As demonstrated by the correlations for LES, which are close to $-1$ (ideal), LES captures the train/test task dissimilarity in FC100 and can be used to anticipate poor performance of meta-learners. On the other hand, CIFAR-FS train and test tasks are more similar, as indicated by the 
%good performance
high accuracy, and thus are harder to distinguish; the correlation is worse for all dataset distances. We report experiment details, run-time comparison, and additional results in Appendix \ref{supp:meta-learning}.

We note some related works. \citet{achille2019task2vec} and \citet{nguyen2021similarity} proposed methods for embedding meta-learning tasks using gradients of an auxiliary neural network and a specialized topic model correspondingly. Both methods require labeled examples to produce the embeddings, while LES is unsupervised. \citet{venkitaraman2020task} and \citet{zhou2021task} incorporated task similarity into training of meta-learners. Their approaches learn task descriptors during training and require specialized architectures and loss functions. Using LES for similarity-aware training of meta-learners can be a promising direction to explore.

% Both CIFAR-FS and FC100 are derived from the CIFAR100 dataset by splitting the 100 classes intro train, validation, and test classes that are used to sample the corresponding tasks. For evaluation we sample new tasks of varying difficulty by combining classes from the train and test sets (i.e., all classes from train is an easy task, and all classes from test is a hard task).

\section{Conclusion}
We proposed the LES distance, a geometric, data-driven intrinsic distance for unaligned datasets of different dimensions that lower-bounds the log-Euclidean metric between SPD matrices. 
LES provides a  stable and efficient distance measure with capabilities of recovering meaningful structural differences on gene expression data and NN embeddings.
Due to LES's generality, it can be applied to a wide range of other tasks and datasets.
In particular, our results on task similarity in few-shot learning suggest applicability to quantifying distribution shifts. Realistic distribution shifts can cause failures of many existing machine learning algorithms \citep{koh2021wilds,santurkar2020breeds}. Using LES to anticipate and diagnose poor out-of-distribution generalization is an interesting future work direction.
% transfer learning and out-of-distribution detection. Realistic distribution shifts have been demonstrated to cause failures of many existing machine learning algorithms citep, thus considering LES for quantifying distribution shifts is an interesting  can help to preempt poor generalization.
% \flushcolsend

\section*{Acknowledgements}
We thank the anonymous reviewers for their valuable feedback and helpful recommendations.
We also wish to thank Madeleine Udell for useful suggestions.
The MIT Geometric Data Processing group acknowledges the generous support of Army Research Office grants W911NF2010168 and W911NF2110293, of Air Force Office of Scientific Research award FA9550-19-1-031, of National Science Foundation grants IIS-1838071 and CHS-1955697, from the CSAIL Systems that Learn program, from the MIT--IBM Watson AI Laboratory, from the Toyota--CSAIL Joint Research Center, and from a gift from Adobe Systems.
TS acknowledges the generous support of the Schmidt Futures Israeli Women's Postdoctoral Award and the Viterbi Fellowship, Technion.

\newpage
\bibliography{papers}
\bibliographystyle{apalike}

%%%%%%%%%%%%%%%%%%%%%%%%%%%%%%%% Appendix %%%%%%%%%%%%%%%%%%%%%%%%%%%%%%%%
\newpage
\appendix
\onecolumn
\section{Proof of the Second Bound in Proposition \ref{prop:evalest_bound}\label{app:prop2proof}}
We reiterate the claim of the second part of Proposition \ref{prop:evalest_bound} for convenience.
\setcounter{prop}{1}
\begin{prop}
The rank-$K$ approximation of a real symmetric PSD matrix, $\mathbf{W}$, computed according to Algorithm \ref{alg:evalest}, and denoted by $\hat{\mathbf{W}}_K$, satisfies the following error bounds.
\begin{enumerate}
\setcounter{enumi}{1}
\item The $\gamma$-regularized log-eigenvalue error is bounded by: 
\begin{flalign}
    & \mathbb{E}\!\left[\sum_{i=1}^K\left\vert\log\!\left(\lambda_i^{(\mathbf{W})}\!+\!\gamma\right)\!-\!\log\!\left(\lambda_i^{(\hat{\mathbf{W}}_K)}\!+\!\gamma\right)\right\vert\right]\leq\frac{1.5K}{(M-K-1)(\lambda_K^{(\mathbf{W})}+\gamma)}\!\sum_{i=K\!+\! 1}^N\!\!\lambda_i^{(\mathbf{W})}\nonumber
\end{flalign}
\end{enumerate}
for $\tfrac{|\lambda^{(\mathbf{W})}_i-\lambda^{(\hat{\mathbf{W}}_K)}_i|}{\lambda^{(\mathbf{W})}_i+\gamma}\!\leq\! 0.5828$, where the expectations are over the normal distribution (due to the random $\mathbf{\Omega}$ in Algorithm \ref{alg:evalest}), $\lambda_i^{(\mathbf{W})}$ and $\lambda_i^{(\hat{\mathbf{W}}_K)}$ denote the eigenvalues of $\mathbf{W}$ and $\hat{\mathbf{W}}_K$, respectively, organized in decreasing order, and $M$ denotes the number of random vectors (columns) in $\mathbf{\Omega}$.
\end{prop}
\begin{proof}
To bound the log-eigenvalues, we denote the $i$-th approximation error by $\lambda_i^{(\hat{\mathbf{W}}_K)}=\lambda_i^{(\mathbf{W})}\pm\epsilon_i$:
\begin{flalign}
    &\sum_{i=1}^K\left\vert\log\!\left(\lambda_i^{(\hat{\mathbf{W}}_K)}\!+\!\gamma\right)\!-\!\log\!\left(\lambda_i^{(\mathbf{W})}\!+\!\gamma\right)\right\vert=\sum_{i=1}^K\left\vert\log\!\left(\lambda_i^{(\mathbf{W})}\!\pm\!\epsilon_i\!+\!\gamma\right)\!-\!\log\!\left(\lambda_i^{(\mathbf{W})}\!+\!\gamma\right)\right\vert&&\nonumber\\
    =&\sum_{i=1}^K\left\vert\log\!\left(1\pm\frac{\epsilon_i}{\lambda_i^{(\mathbf{W})}\!+\!\gamma}\right)\right\vert
    \overset{(*)}{\leq} 1.5\sum_{i=1}^K\left\vert\frac{\epsilon_i}{\lambda_i^{(\mathbf{W})}\!+\!\gamma}\right\vert\overset{(**)}{\leq}\! 1.5\sum_{i=1}^K\left\vert\frac{\epsilon_i}{\lambda_K^{(\mathbf{W})}\!+\!\gamma}\right\vert\!=\!\frac{1.5}{\lambda_K^{(\mathbf{W})}\!+\!\gamma}\sum_{i=1}^K\left\vert\lambda_i^{(\mathbf{W})}-\lambda_i^{(\hat{\mathbf{W}}_K)}\right\vert&&\label{eq:logevalest_proof}
\end{flalign}
where transition $(*)$ is due to $\log(1\!+\!x)\!\leq\! x$ for $x\!\geq\! 0$ and $|\log(1\!-\!x)|\!\leq\! 1.5x$ for $0\!\leq\! x\!\leq\! 0.5828$, and $(**)$ is since the eigenvalues $\lambda^{(\mathbf{W})}_i$ are positive and decreasing. 
Therefore, this derivation holds when $\tfrac{|\lambda^{(\mathbf{W})}_i-\lambda^{(\hat{\mathbf{W}}_K)}_i|}{\lambda^{(\mathbf{W})}_i+\gamma}\!\leq\! 0.5828$.
Taking the expectation of \eqref{eq:logevalest_proof} and bounding it with $\eqref{eq:eigbound}$ yields the error bound in the proposition.
\end{proof}

\section{Tori Simulation: Additional Details and Results\label{app:tori_exp}}
\subsection{Tori Equations and Parameter Choices}
The tori embedding equations used to generate the data for the toy example in Subsection \ref{sub:exp_tori} are given by:
\begin{flalign}
    &T_{2}:
    \begin{Bmatrix}
    x[i] = & (R_1+R_2\cos(\theta_2[i]))\cos(\theta_1[i])\\ 
    y[i] = & (R_1+R_2\cos(\theta_2[i]))\sin(\theta_1[i])\\
    z[i] = & R_2\sin(\theta_2[i])
    \end{Bmatrix}\nonumber\\
    &T_{3}:\begin{Bmatrix}
    x[i] = & (R_1+(R_2+R_3\cos(\theta_3[i]))\cos(\theta_2[i]))\cos(\theta_1[i])\\ 
    y[i] = & (R_1+(R_2+R_3\cos(\theta_3[i]))\cos(\theta_2[i]))\sin(\theta_1[i])\\
    z[i] = & (R_2+R_3\cos(\theta_3[i]))\sin(\theta_2[i])\\
    w[i] = & R_3\sin(\theta_3[i])
    \end{Bmatrix}\label{appeq:tori_eqs}
\end{flalign}
where $\theta_1[i],\theta_2[i],\theta_3[i]\sim \mathrm{U}[0,2\pi]$, $i=1,\dots N$ and $R_1=10,\ R_2=3,\ R_3=1$.
$T_2^{Sc}$ was defined similarly to $T_2$ by replacing $R_2$ with $R_2^{Sc}=c\!\cdot\!R_2$ and $T_3^{Sc}$ was defined by replacing $R_3$ in $T_3$ with $R_3^{Sc}=c\!\cdot\!R_3$. 
Figure \ref{appfig:tori_vis} presents examples of data sampled from the $2$D torus and the scaled $2$D torus.

\begin{figure*}[ht]
\centering
\subfigure[]{
    \includegraphics[width=0.25\textwidth]{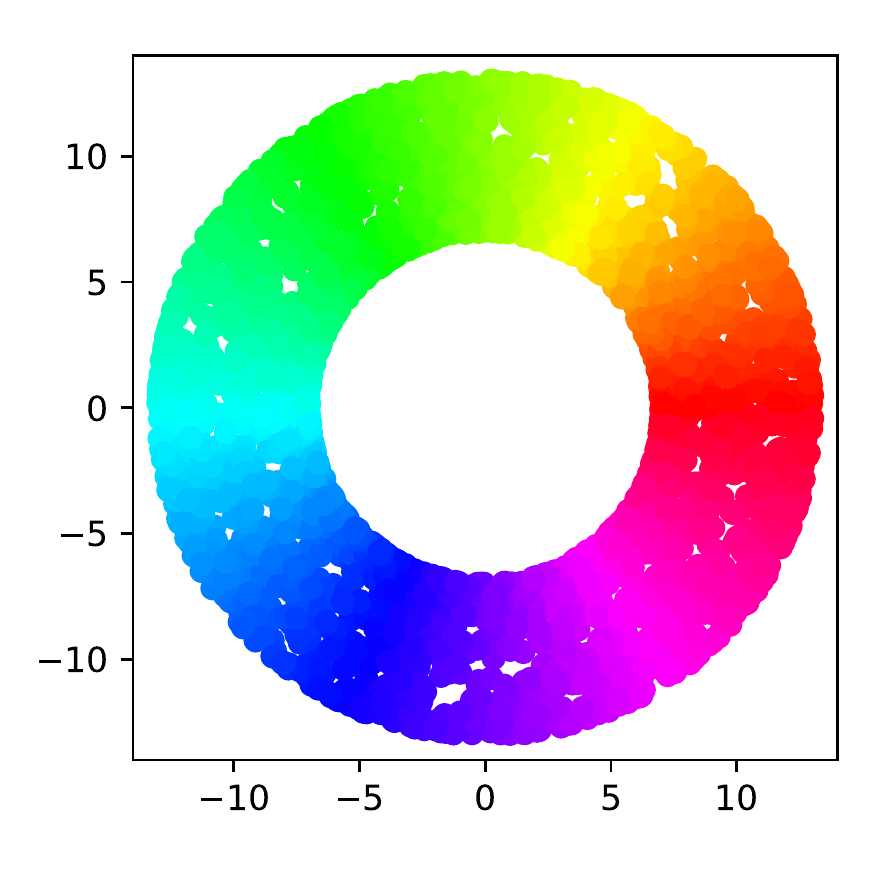}
}
\hspace{-0.2in}
\subfigure[]{
    \includegraphics[width=0.25\textwidth]{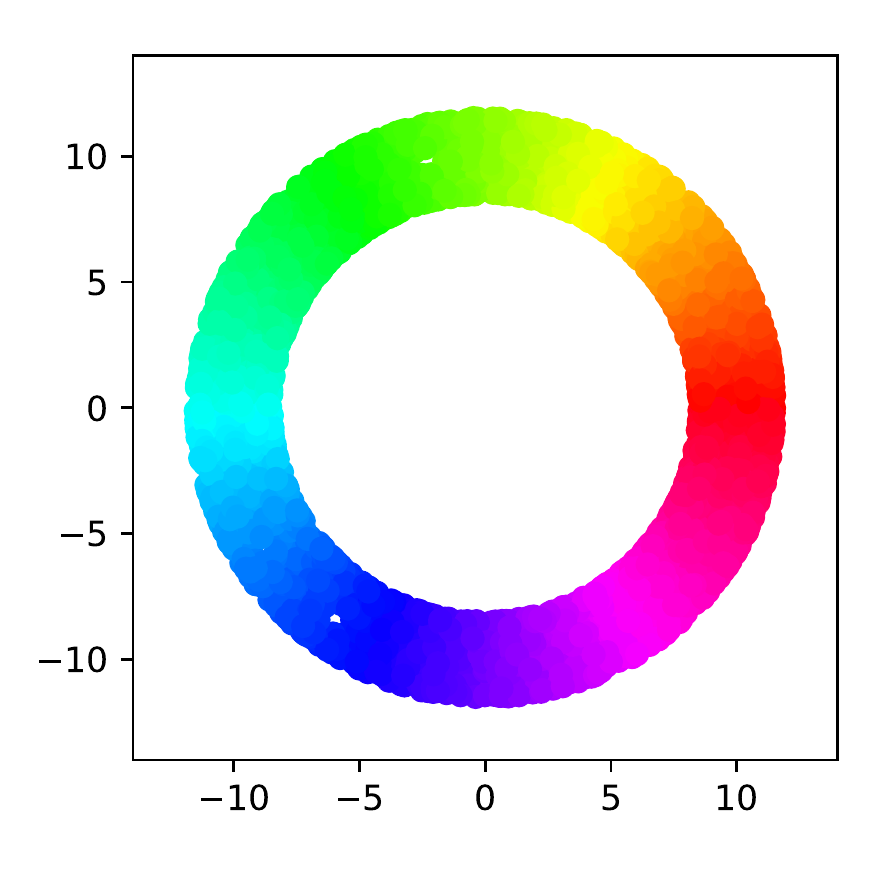}
}
\caption{Visualization of (a) the $2$D torus and (b) the scaled $2$D torus with $c=0.5$.\label{appfig:tori_vis}}
\end{figure*}

For \cref{fig:tori_dist}, we sampled $N=3000$ points of $\theta_1,\ \theta_2$ and $\theta_3$ in each torus independently, such that no known pointwise correspondence exists.
In our comparisons we used the default parameters for GW, IMD with the exact $k$-nearest neighbors distance matrix and GS with $n=2000$.
Using non-default parameters did not lead to significant differences in the results.

\textbf{LES hyperparameters.} In the computation of LES, we set $K=200$ and $\gamma=10^{-8}$ since the simulated data is clean and the eigenvalues decay quickly.
In all experiments, the choice of $\gamma$ was motivated by the magnitude of the smallest eigenvalues.

\subsection{Comparison to Persistence Diagram Metrics\label{appsub:tori_topo}}
As additional examples of the differences between our approach and topological data analysis methods, we compare LES to the bottleneck distance between $H_0$, $H_1$, and $H_2$ persistence diagrams \citep{edelsbrunner2008persistent} on the tori datasets.
We compute the $H_0$, $H_1$, and $H_2$ persistence diagrams \citep{ctralie2018ripser} for samples from each of the $4$ tori and then compute the pairwise bottleneck distances between them. 
We sample only $N=1000$ points from each torus in this example due to the long run times of the persistence diagram computations.
Figure \ref{appfig:tori_bottleneck} presents the average and standard deviation of the distances (avg.\ over 10 trials) as a function of the scale parameter $c$, similarly to \cref{fig:tori_dist}, along with LES for comparison.
While the $H_0$ and $H_1$ bottleneck distances do capture meaningful differences between the tori, as expected, LES distinguishes between dimensionality and scale differences for a wider range of $c$ values.

\begin{figure*}[ht]
\centering
\subfigure{
\includegraphics[width=0.6\textwidth]{figures/Tori/Tori_dist_legend.pdf}
}\vspace{-0.15in}

\setcounter{subfigure}{0}
\subfigure[LES]{
    \includegraphics[width=0.23\textwidth]{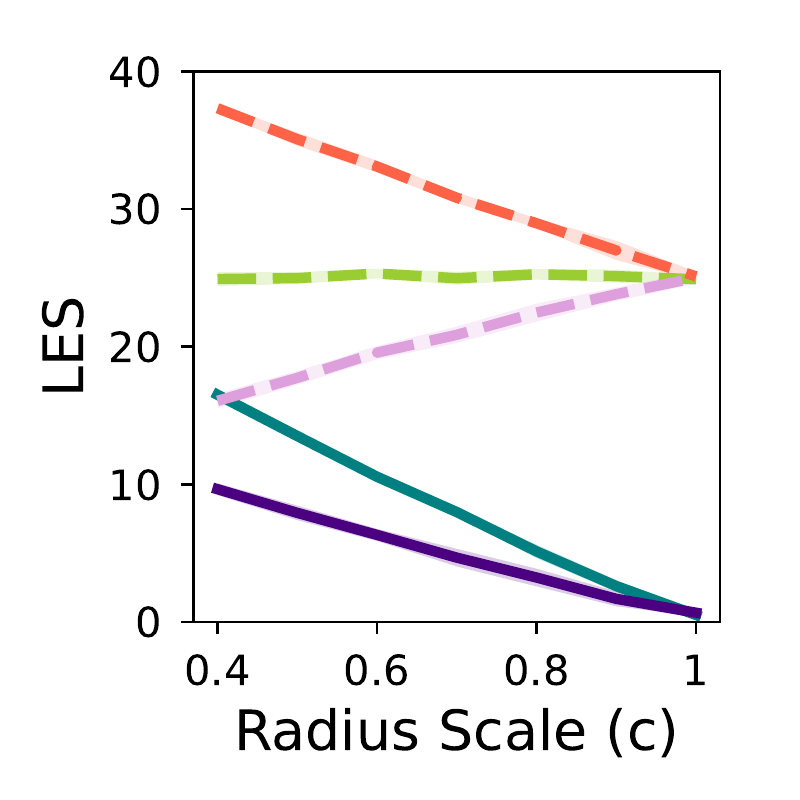}
}
% \hspace{-0.2in}
\subfigure[$H_0$, $H_1$ and $H_2$ bottleneck distances]{
    \includegraphics[width=0.23\textwidth]{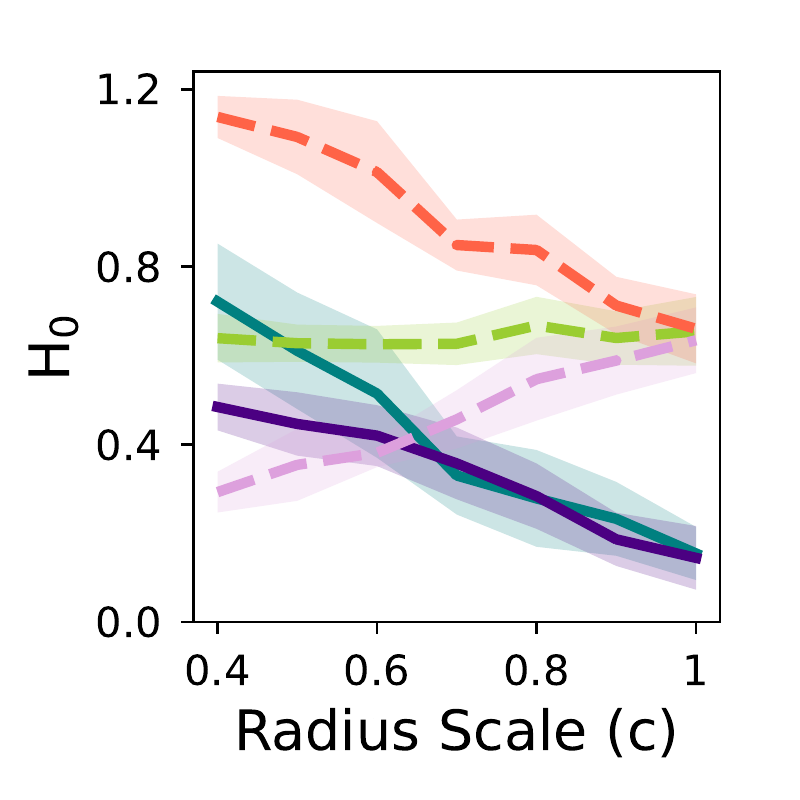}
    \includegraphics[width=0.23\textwidth]{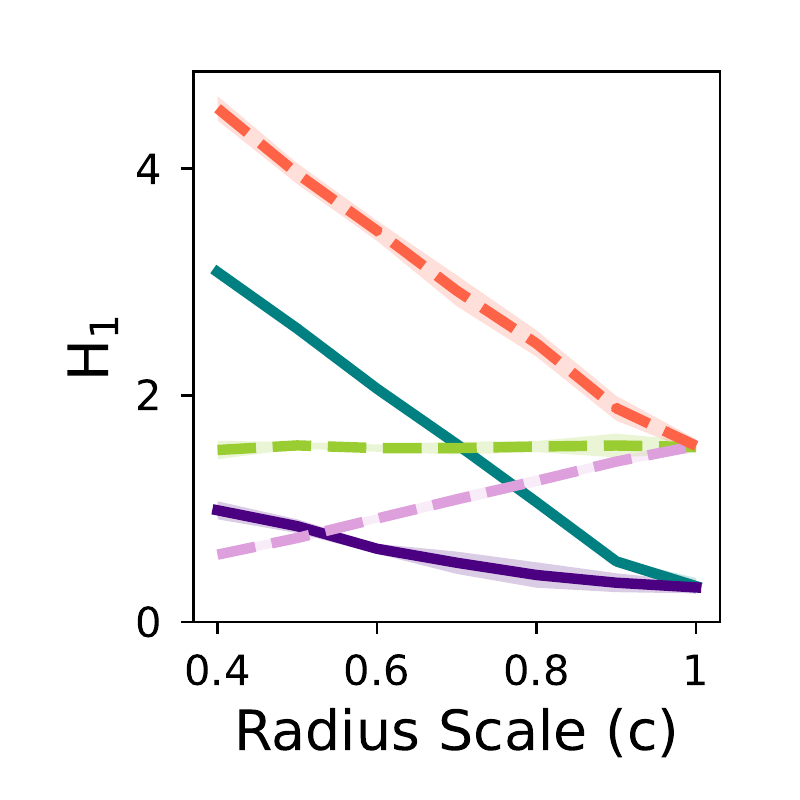}
    \includegraphics[width=0.23\textwidth]{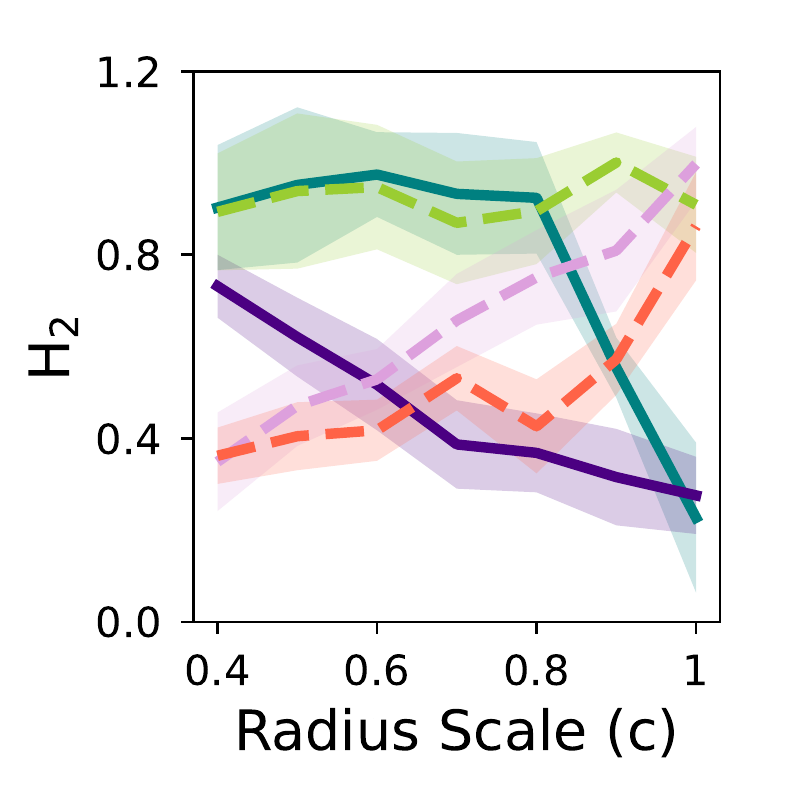}
}
\caption{LES vs. $H_0$, $H_1$ and $H_2$ bottleneck distances on the $2$D and $3$D tori example.\label{appfig:tori_bottleneck}}
\end{figure*}

\subsection{Comparisons to the Log-Euclidean Distance Between Aligned Datasets\label{appsub:tori_addres}}
To evaluate the use of Riemannian SPD geometry in our setting and the relation between the log-Euclidean (LE) distance and LES,
% For additional evaluation of the LES distance, its relation to the log-Euclidean (LE) distance between SPD matrices and the advantage of accounting for the Riemannian geometry of the SPD matrix space, 
we generated the tori datasets \emph{with} pointwise correspondence by using the same samples of $\theta_1,\ \theta_2$ and $\theta_3$ for all tori.
We constructed the SPD matrix $\mathbf{W}$ \eqref{eq:dm_spd} for each torus and computed the LE distance, the affine-invariant (AI) distance, $d_{AI}\left(\mathbf{W}_1,\mathbf{W}_2\right)=\left\Vert\log\left(\mathbf{W}_1^{-1}\mathbf{W}_2\right)\right\Vert_F$, the Euclidean distance, $d_{EU}\left(\mathbf{W}_1,\mathbf{W}_2\right)=\left\Vert\mathbf{W}_1-\mathbf{W}_2\right\Vert_F$, and the Spectral Gromov-Wasserstein (SpecGW) distance \citep{memoli2011spectral}, $d_{SpecGW}\left(\mathbf{W}_1,\mathbf{W}_2\right)=\sup_{t>0}e^{-(t+t^{-1})}\left\Vert\mathbf{W}_1^t-\mathbf{W}_2^t\right\Vert$ (for known correspondence). All distances were computed using the ground truth correspondence.
To avoid violations of the SPD structure of $\mathbf{W}$ due to numerical issues, we limit the number of samples on the tori to $N=100$ in this comparison.
Recall that LES is a lower bound of the LE distance and that IMD is a lower bound of SpecGW.
\Cref{appfig:tori_le_sgw} presents the average and standard deviation of the computed distances (avg.\ over 10 trials) between the tori for different values of $c$, similarly to \cref{fig:tori_dist}. 
This figure demonstrates that both the LE and AI distances follow a similar trend and are able to distinguish between scale and dimensionality differences for a wider range of $c$ values, compared with the Euclidean and SpecGW distances.
Note that both LES and IMD (our approach), \cref{fig:tori_les,fig:tori_imdours}, approximately follow the same trends as LE and SpecGW.
This justifies our approach for approximating and comparing the eigenvalues of the heat kernel and indicates that our lower bound of the LE distance indeed captures similar structures as the LE distance in this example.
Note that the values of SpecGW are smaller than IMD (our approach) in \cref{fig:tori_imdours}. This may be due to the truncation of the spectrum in IMD (our approach).

\begin{figure*}[ht]
\centering
\subfigure{
\includegraphics[width=0.6\textwidth]{figures/Tori/Tori_dist_legend.pdf}
}\vspace{-0.15in}

\setcounter{subfigure}{0}
\subfigure[LE]{
    \includegraphics[width=0.22\textwidth]{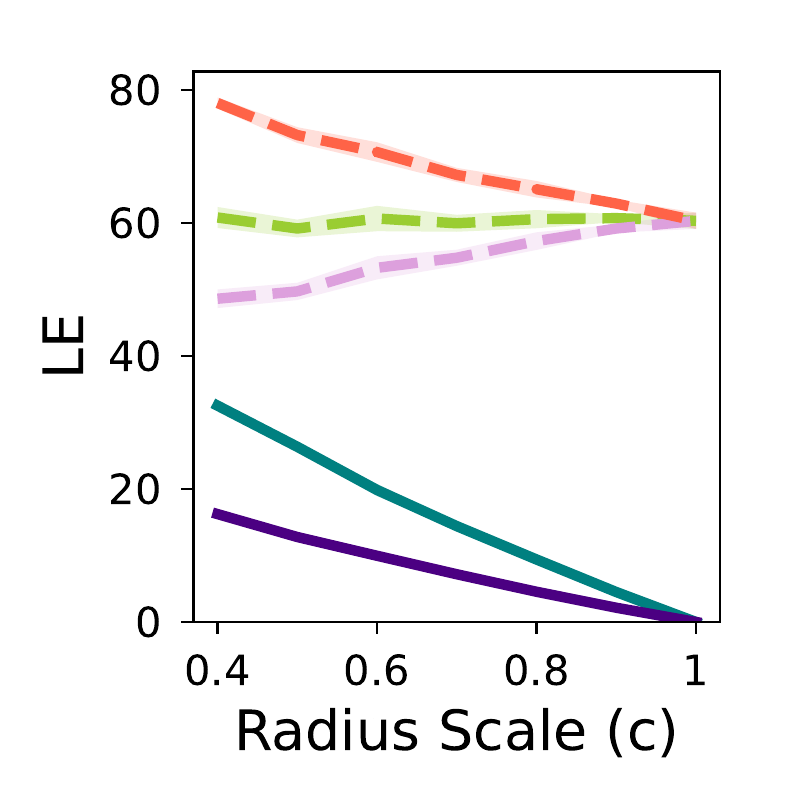}
    \label{appfig:tori_le}
}
\subfigure[AI]{
    \includegraphics[width=0.22\textwidth]{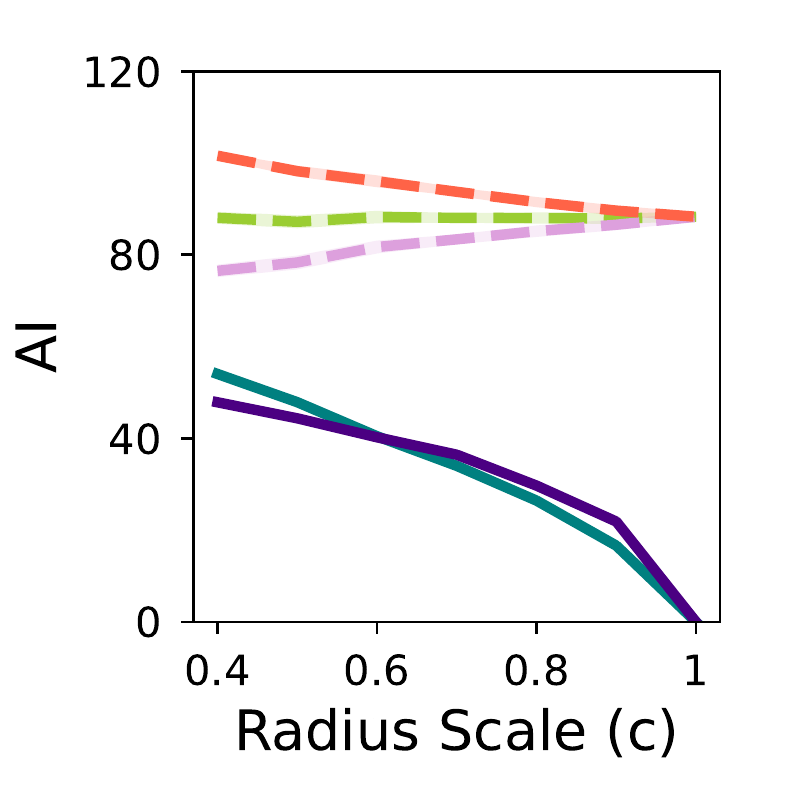}
    \label{appfig:tori_ai}
}
\subfigure[Euclidean]{
    \includegraphics[width=0.22\textwidth]{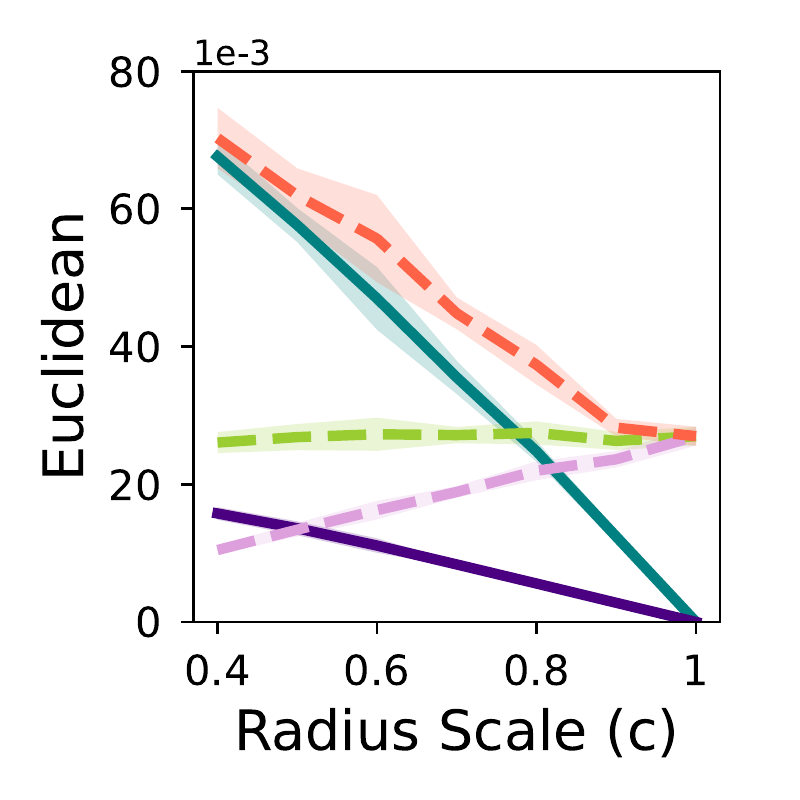}
    \label{appfig:tori_eu}
}
% \hspace{-0.2in}
\subfigure[Spectral GW]{
    \includegraphics[width=0.22\textwidth]{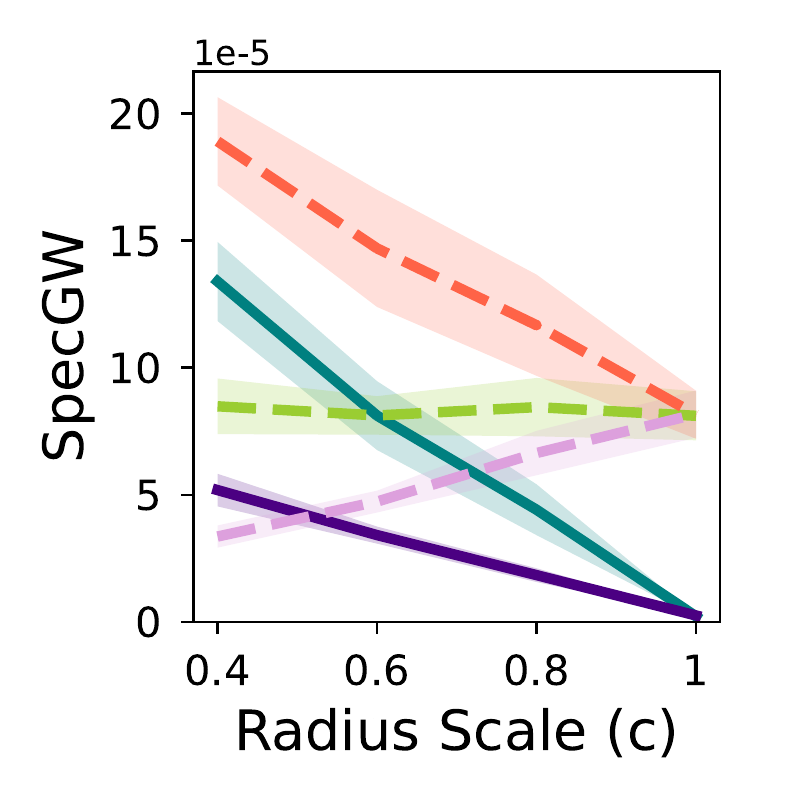}
    \label{appfig:tori_sgw}
}
\caption{LE, AI, Euclidean and SpecGW distances on the $2$D and $3$D tori example using true pointwise correspondence (coupled angles).\label{appfig:tori_le_sgw}}
\end{figure*}

\section{Gene Expression Data: Additional Details\label{app:gene_exp}}
\Cref{appfig:gene_vis} presents the gene space of the scRNA-seq data used in Subsection \ref{sub:exp_genes}, colored according to the time of collection. At each time point, gene expressions from $2000\mbox{--}12000$ cells were collected. 
This figure was created using code provided by \citet{schiebinger2019optimal}, which applies force-directed layout embedding \citep{jacomy2014forceatlas2} for the visualization.

In this experiment, we compare the cell group structures from different time points, treating each group as a separate dataset. 
For LES, we construct the operator $\mathbf{W}_\ell$ at each time point, $\ell=1,\dots,39$, and then compare the descriptors, $F_{LES}^{(\ell)}$, for every pair of time points.

\begin{figure}[ht]
\centering
\includegraphics[width=0.4\textwidth]{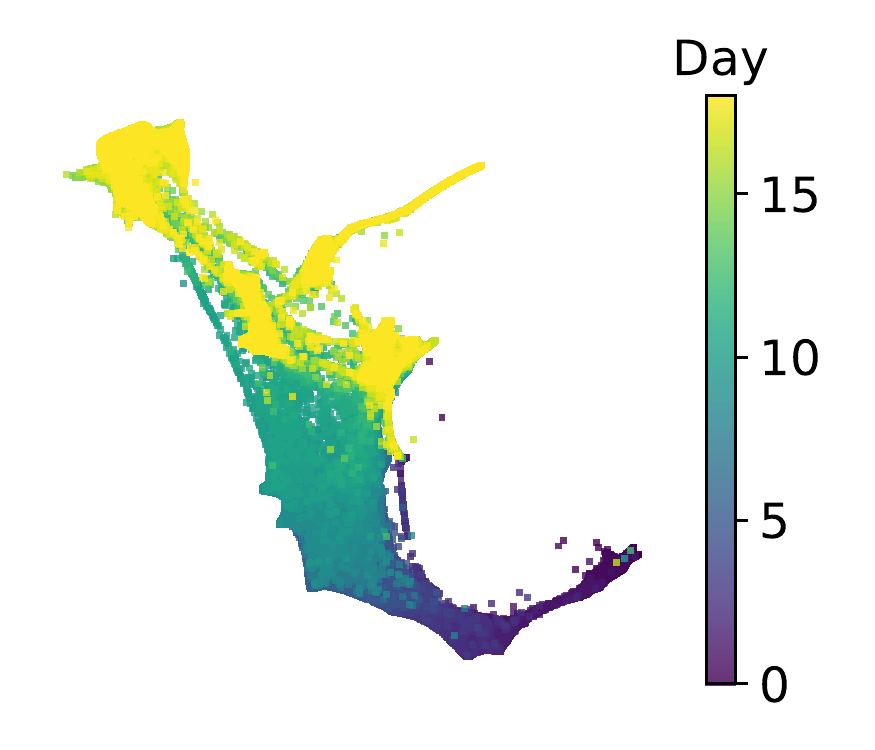}
\caption{Gene expression data visualization \citep{schiebinger2019optimal}.\label{appfig:gene_vis}}
\end{figure}

\textbf{LES hyperparameters.} In the gene expression data we set the LES parameters to $K=500$ and $\gamma=10^{-5}$, which is slightly larger than the maximal $\lambda_{500}$ in all time points, to reduce the effect of noise in the data.

\section{Identifying Difficult Tasks in Few-Shot Learning: Additional Details and Results}
\label{supp:meta-learning}
% How tasks are sampled and score computation details; other aggregation functions. Run times. Task embeddings.

We present details regarding the experiment described in Section \ref{sec:meta-learning} and summarized in Table \ref{table:meta-mean}.

\textbf{Benchmark datasets.} Both CIFAR-FS \citep{bertinetto2018meta} and FC100 \citep{oreshkin2018tadam} are few-shot learning benchmarks derived from the CIFAR-100 dataset \citep{krizhevsky2009learning}. First the 100 classes are disjointly partitioned intro train, validation, and test. CIFAR-FS partitions the classes randomly, while FC100 utilizes the class similarity information (100 classes are grouped into 20 superclasses in CIFAR-100) to reduce the semantic overlap between train and test classes. Due to the partitioning, FC100 is a more challenging benchmark making it especially interesting for our study. Then the training tasks are obtained by randomly selecting classes and the corresponding images from the set of training classes. Similarly, test tasks are sampled from the set of test classes. All tasks in our experiments are 5-way, i.e. consist of 5 unique classes, following standard practice for these benchmarks.

\textbf{Feature extractor and test accuracies.} For each benchmark and shot setting we train MetaOptNet-SVM following the original implementation \citep{lee2019meta}. Each MetaOptNet network consists of a ResNet-12 feature extractor that is shared across tasks during evaluation and a linear support vector machine classification ``head'' that is trained on $k$ labeled samples, the support (adaptation) data, for each task to compute accuracies, where $k$ denotes the number of ``shots'' per class.
%of the corresponding size (number of ``shots'' per class) for each task to compute accuracies. 
Test accuracies (first row in Table \ref{table:meta-mean}) in each setting are computed using 1000 random test tasks, and we report mean and standard deviation.

\textbf{Evaluating correlations.} To evaluate correlations we first sample $n=10$ tasks from the train task distribution. Corresponding images are then passed through the MetaOptNet feature extractor resulting in a collection of $n$ train task embeddings $\{X^{(\ell)} \in \mathbb{R}^{N_\ell \times d}\}_{\ell=1}^n$, $d=2560$. To obtain a range of new tasks from easier to more difficult ones, we consider mixtures of train and test classes: only test classes (most difficult), 4 test and 1 train classes, 3 test and 2 train classes, 2 test and 3 train classes, 1 test and 4 train classes, and only train classes (easiest). For each of the six combinations we sample 3 tasks, resulting in 18 new tasks for computing the correlations. For each new task we obtain its embeddings, $X' \in \mathbb{R}^{N' \times d}$, compute accuracy using $k$ (number of shots) random labeled examples per class, and compute a measure of dissimilarity with the train tasks, $\mathcal{A}(\{d(X^{(\ell)}, X')\}_{\ell=1}^n)$. For dataset distance computations, $d(\cdot,\cdot)$, we use all available (unlabeled) images, i.e. $N_\ell = N' = N = 3000$ $\forall l$ (600 images per each of the 5 classes in a task). We note that in practice this would require collecting unlabeled examples from a new task to evaluate the dissimilarity, however it is a minor limitation due to easy availability of unlabeled data in many application domains.

For a dataset distance that is informative of the task difficulty, we expect dissimilarity values and accuracies to be negatively correlated, i.e., large distances should indicate low accuracies. An informative dataset distance can be used in practice to anticipate difficult tasks before collecting labels and to prevent deployment of potentially poor-performing models. To compare different dataset distances we report Pearson correlation coefficient between accuracies and dissimilarity measures evaluated on 18 new tasks sampled as previously described. We considered three choices of the aggregation function $\mathcal{A}(\cdot)$: mean (Table \ref{table:meta-mean}), min (Table \ref{table:supp:corr-min-metalearn}), and max (Table \ref{table:supp:corr-max-metalearn}). We repeat each experiment 20 times and report mean and standard deviation of the correlations. Comparing the results across tables, we note that the choice of the aggregation function does not alter the performance much. Our proposed LES is the overall best performing dataset distance regardless of the aggregation function choice.

\paragraph{Run time comparison.} In Table \ref{table:supp:time-metalearn} we summarize run times for computing the dissimilarity measures based on each of the dataset distances considered in this experiment. We report mean and standard deviation over 360 dissimilarity computations (18 tasks for each of the 20 experiment repetitions). IMD and LES are the fastest, with IMD being slightly faster overall. To compute the dissimilarity measure for a new task we need to evaluate each dataset distance $n=10$ times, i.e.\ one evaluation for each of the reference train tasks. LES and IMD are more efficient because they allow pre-computing train task descriptors and reusing them with every new task. Thus, their dissimilarity evaluation run-time is effectively equal to the run-time of computing task descriptor for a new task (the run time of $n=10$ Euclidean distance computations with the pre-computed train task descriptors is negligible). On the other hand, dataset distances such as OT and GW do not have a notion of task descriptors and require $n=10$ respective dataset distance evaluations to compute the dissimilarity measure. While GW is comparable to LES in terms of the correlation scores, its runtime might be problematic in practice.

\paragraph{Why does LES work well?} We present a hypothesis explaining the strong correlation between LES-based dissimilarity measure and task accuracies. MetaOptNet uses linear SVM for classification, and thus its embeddings need to be linearly separable with respect to classes for good performance. In Figure \ref{fig:supp:mds-metalearn} (left), we present 2-dimensional visualization of MetaOptNet embeddings obtained with Multidimensional Scaling \citep{mead1992review} on one of the FC100 5-shot train tasks. While not guaranteed in general without the label information, it is likely that the 5 well-separated clusters we see correspond to the 5 classes from this task. Accuracy on this task is 98.2\%, and we hypothesize that if the embeddings on a new task look similar, i.e. have number of clusters corresponding to the number of classes, then it is likely that the corresponding few-shot accuracy will be high. In Figure \ref{fig:supp:mds-metalearn} (center) and (right) we show embeddings of two new tasks with distinct few-shot accuracies: 66.2\% for the (center) embeddings without cluster structure, and 96.0\% for the (right) embeddings with pronounced cluster structure. Evidently, the geometric structure of the (left) and (right) embeddings is a lot more similar than that of the (left) and (center) embeddings, which is captured by LES, but not IMD. Specifically, LES(left, center) = 188.6 $\gg$ LES(left, right) = 3.5; IMD(left, center) = 3.5 $<$ IMD(left, right) = 12.6. LES takes into account cluster structure when comparing datasets, leading to strong correlation between LES-based dissimilarity and few-shot accuracy in our experiments.

\textbf{LES hyperparameters.} In these experiments we set the LES parameters to $K=500$ and $\gamma=10^{-6}$. Other values of $\gamma$ were tested and did not significantly affect the results.

\begin{table}[t]
\caption{Few-shot learning: MetaOptNet test accuracies and correlations (20 trials) with task dissimilarities using \textbf{min}-aggregation.}
\label{table:supp:corr-min-metalearn}
\begin{center}
\scalebox{.85}{
\begin{tabular}{lcccccc}
\toprule
{} &               FC100 1-shot &               FC100 5-shot &              FC100 10-shot &            CIFAR-FS 1-shot &            CIFAR-FS 5-shot &           CIFAR-FS 10-shot \\
\midrule
Acc. & 38.19$\pm$0.48\%   &    54.45$\pm$0.49\%     &    60.52$\pm$0.48\%     &    70.79$\pm$0.69\%  &   83.98$\pm$0.44\%  &   87.11$\pm$0.40\% \\
\midrule
LES &  \textbf{-0.913}$\pm$0.038 &  \textbf{-0.922}$\pm$0.042 &  \textbf{-0.913}$\pm$0.046 &  \textbf{-0.588}$\pm$0.150 &  \textbf{-0.676}$\pm$0.131 &  \textbf{-0.597}$\pm$0.203 \\
IMD &           -0.356$\pm$0.258 &           -0.399$\pm$0.223 &           -0.359$\pm$0.203 &           -0.116$\pm$0.196 &           -0.011$\pm$0.221 &           -0.044$\pm$0.272 \\
OT  &            0.314$\pm$0.264 &            0.099$\pm$0.253 &            0.077$\pm$0.267 &            0.243$\pm$0.205 &           -0.013$\pm$0.283 &           -0.031$\pm$0.277 \\
GW  &           -0.898$\pm$0.036 &           -0.869$\pm$0.090 &           -0.831$\pm$0.166 &           -0.525$\pm$0.309 &           -0.506$\pm$0.272 &           -0.388$\pm$0.284 \\
GS  &           -0.744$\pm$0.080 &           -0.546$\pm$0.241 &           -0.555$\pm$0.245 &           -0.306$\pm$0.248 &           -0.094$\pm$0.230 &           -0.051$\pm$0.257 \\
\bottomrule
\end{tabular}
}
\end{center}
% \vspace{-.2in}
\end{table}

\begin{table}[t]
\caption{Few-shot learning: MetaOptNet test accuracies and correlations (20 trials) with task dissimilarities using \textbf{max}-aggregation.}
\label{table:supp:corr-max-metalearn}
\begin{center}
\scalebox{.85}{
\begin{tabular}{lcccccc}
\toprule
{} &               FC100 1-shot &               FC100 5-shot &              FC100 10-shot &            CIFAR-FS 1-shot &            CIFAR-FS 5-shot &           CIFAR-FS 10-shot \\
\midrule
Acc. & 38.19$\pm$0.48\%   &    54.45$\pm$0.49\%     &    60.52$\pm$0.48\%     &    70.79$\pm$0.69\%  &   83.98$\pm$0.44\%  &   87.11$\pm$0.40\% \\
\midrule
LES &  \textbf{-0.938}$\pm$0.032 &  \textbf{-0.942}$\pm$0.029 &  \textbf{-0.931}$\pm$0.029 &           -0.663$\pm$0.167 &  \textbf{-0.690}$\pm$0.170 &  \textbf{-0.644}$\pm$0.194 \\
IMD &           -0.119$\pm$0.269 &           -0.298$\pm$0.223 &           -0.178$\pm$0.284 &           -0.129$\pm$0.194 &           -0.026$\pm$0.247 &           -0.028$\pm$0.236 \\
OT  &            0.754$\pm$0.096 &            0.634$\pm$0.124 &            0.621$\pm$0.146 &            0.432$\pm$0.184 &            0.283$\pm$0.231 &            0.184$\pm$0.220 \\
GW  &           -0.924$\pm$0.034 &           -0.915$\pm$0.052 &           -0.908$\pm$0.052 &  \textbf{-0.679}$\pm$0.146 &           -0.660$\pm$0.166 &           -0.610$\pm$0.171 \\
GS  &           -0.705$\pm$0.100 &           -0.517$\pm$0.170 &           -0.517$\pm$0.175 &           -0.407$\pm$0.207 &           -0.073$\pm$0.252 &           -0.005$\pm$0.281 \\
\bottomrule
\end{tabular}
}
\end{center}
% \vspace{-.2in}
\end{table}

\begin{table}[t]
\caption{Few-shot learning: run times (in seconds) for evaluating new task dissimilarity score using $n=10$ train tasks (360 trials).}
\label{table:supp:time-metalearn}
\begin{center}
\scalebox{.85}{
\begin{tabular}{lcccccc}
\toprule
{} &            FC100 1-shot &            FC100 5-shot &           FC100 10-shot &         CIFAR-FS 1-shot &         CIFAR-FS 5-shot &       CIFAR-FS 10-shot \\
\midrule
LES &           23.3$\pm$27.4 &           20.4$\pm$11.8 &  \textbf{19.2}$\pm$11.3 &           26.2$\pm$27.2 &  \textbf{20.8}$\pm$11.6 &          18.8$\pm$10.6 \\
IMD &  \textbf{16.6}$\pm$11.1 &  \textbf{16.1}$\pm$11.5 &           24.1$\pm$28.6 &  \textbf{18.4}$\pm$11.3 &           22.9$\pm$27.7 &  \textbf{16.7}$\pm$9.9 \\
OT  &           29.2$\pm$10.7 &           31.0$\pm$11.1 &           40.3$\pm$41.3 &           31.0$\pm$10.3 &           39.3$\pm$41.6 &          28.2$\pm$10.3 \\
GW  &        1569.3$\pm$480.4 &        1591.8$\pm$474.5 &        1716.2$\pm$435.7 &        1458.9$\pm$438.9 &        1539.8$\pm$448.6 &       1579.0$\pm$453.5 \\
GS  &            96.2$\pm$7.7 &            97.2$\pm$7.9 &            99.2$\pm$6.5 &            99.8$\pm$7.4 &            98.3$\pm$7.5 &           97.6$\pm$8.6 \\
\bottomrule
\end{tabular}
}
\end{center}
% \vspace{-.2in}
\end{table}

\begin{figure}[ht]
\centering
\includegraphics[width=\textwidth]{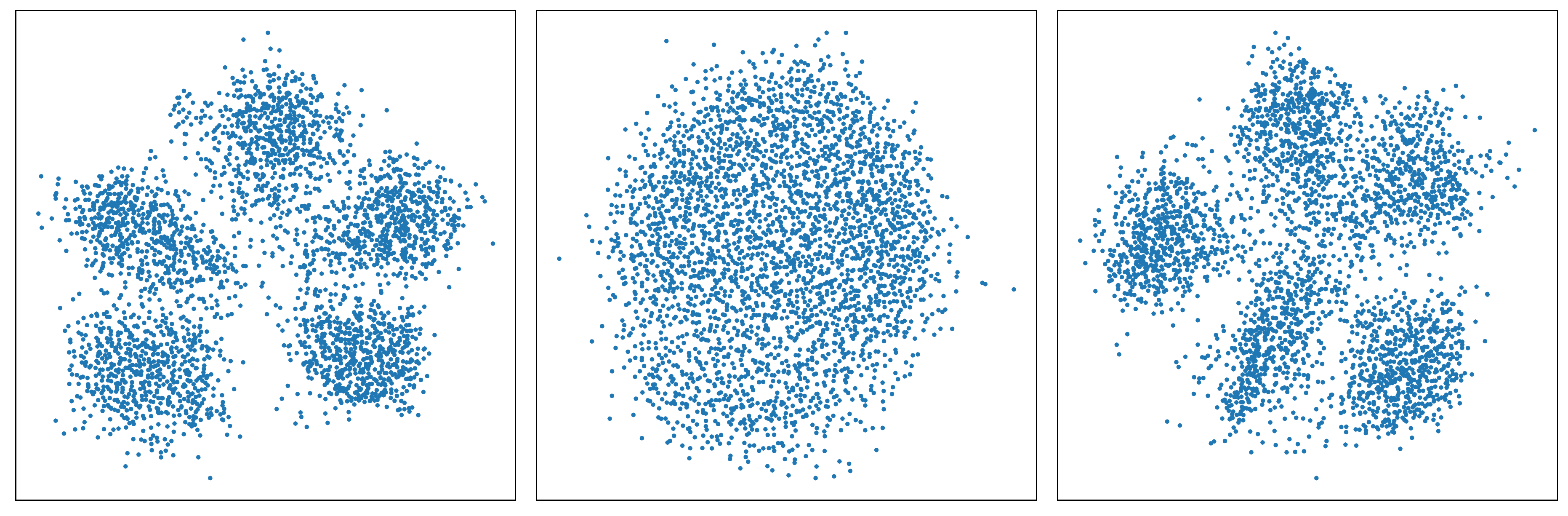}
\vskip -0.2in
\caption{Projected MetaOptNet FC100 5-shot embeddings on one train (left) and two new tasks (center and right). The corresponding few-shot accuracies are 98.2\%, 66.2\%, and 96.0\%. LES(left, center) = 188.6,  LES(left, right) = 0.5; IMD(left, center) = 3.5, IMD(left, right) = 12.6. LES better captures the geometry of the data. Specifically, LES is perceptive of the cluster structure in the embeddings which we hypothesize to be the reason for its success in anticipating difficult tasks in few-shot learning.}
\label{fig:supp:mds-metalearn}
\end{figure}

\section{Matching Neural Network Layers}\label{app:layers}
In this experiment, we demonstrate the applicability of LES to measuring similarity of NN hidden representations, inspired by \citet{kornblith2019similarity}.

The hidden representations of NN are intrinsically difficult to compare, since they can differ in dimensionality and may be invariant to a variety of transformations, such as permutation of the hidden units.
Therefore, the intrinsic properties of the LES distance facilitate the comparison of different layer embeddings.

We train a LeNet convolutional NN on the MNIST training set using two different random seeds (seed 0 and seed 1).
The MNIST test set is then passed through the two trained NN and the resulting layer embeddings at the $i$th layer are saved as $X_i^{(0)}$ and $X_i^{(1)}$, where $i \in \{0,\dots, 4\}$. Here, $i=0$ denotes the input, and we omit the output of the last layer. 
We compute the pairwise distances between layers of the two NN, $d(X_i^{(0)},X_j^{(1)})\ \ \forall i,j\in \{0,\dots, 4\}$, using LES (with $K=200$ and $\gamma=10^{-8}$) and IMD.
We denote these distance matrices by $\mathbf{D}^{\mathrm{LES}}$ and $\mathbf{D}^{\mathrm{IMD}}$.
Due to the similar structure of the two NNs, we expect that hidden representations of the layers will be consistent across the random initializations, resulting in a distance matrix with minimal diagonal values.
Indeed, $\mathbf{D}^{\mathrm{LES}}$ depicts this expected behavior, whereas in $\mathbf{D}^{\mathrm{IMD}}$ some of the off-diagonal entries are smaller than the diagonal entries.
To demonstrate this different behavior, we define the margin factor $\alpha_\delta$ for each layer shift $\delta\in[1,4]$ by 
$$ \alpha_\delta = \frac{\min_{j,k\ s.t.\ \vert k-j\vert=\delta} D_{jk}}{\max_i D_{ii}}$$
where $D_{jk}$ denotes the $(j,k)$th entry in either $\mathbf{D}^{\mathrm{LES}}$ or $\mathbf{D}^{\mathrm{IMD}}$, and we set $\alpha_0=1$.
Namely, we compute the ratio between the minimal distance of layers shifted by $\delta$ and the maximal distance of layers from the same level in the two random initializations.
Figure \ref{appfig:margin} presents the $\log$ of the margin factor for LES and IMD as a function of the layer shift, $\delta$.
The log-margin factor computed based on the LES distance exhibits the expected behavior (the log-margin factor is greater than $0$ ($\log(1)$) for all layer shifts, $\delta$, and increases as $\vert\delta\vert$ grows), whereas the IMD plot dips below $0$ for $\delta=\pm 1$.
Both methods have margin curves that monotonically increase away from $\delta=\pm 1$, indicating that nearby layers are significantly more similar than layers further apart.

\begin{figure}[ht]
\centering
\subfigure[LES]{
    \includegraphics[width=0.25\textwidth]{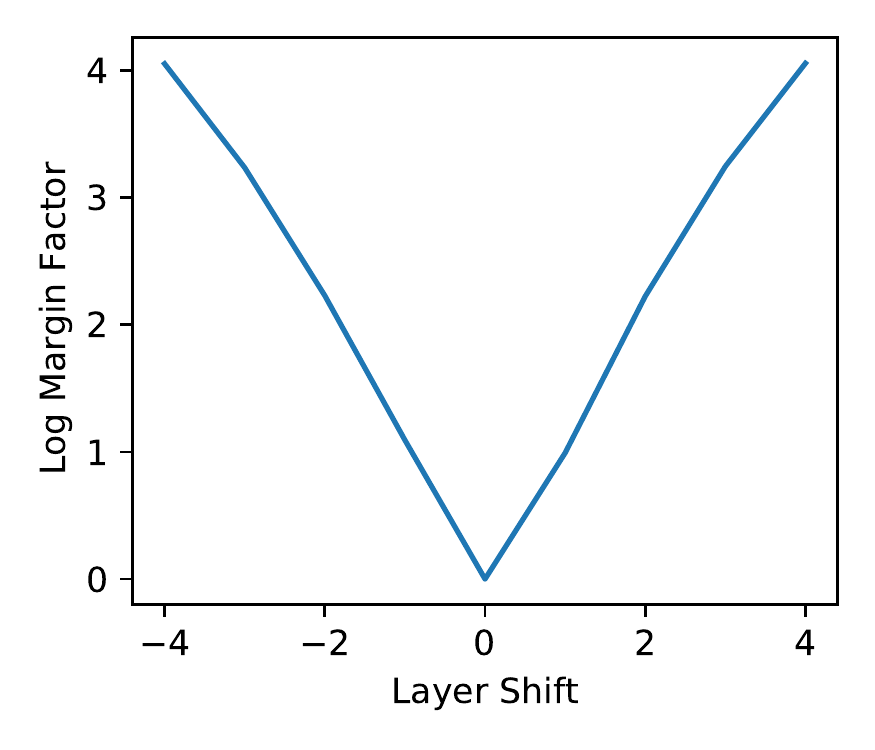}
    \label{appfig:Margin_le}
}
\hspace{-0.2in}
\subfigure[IMD]{
    \includegraphics[width=0.25\textwidth]{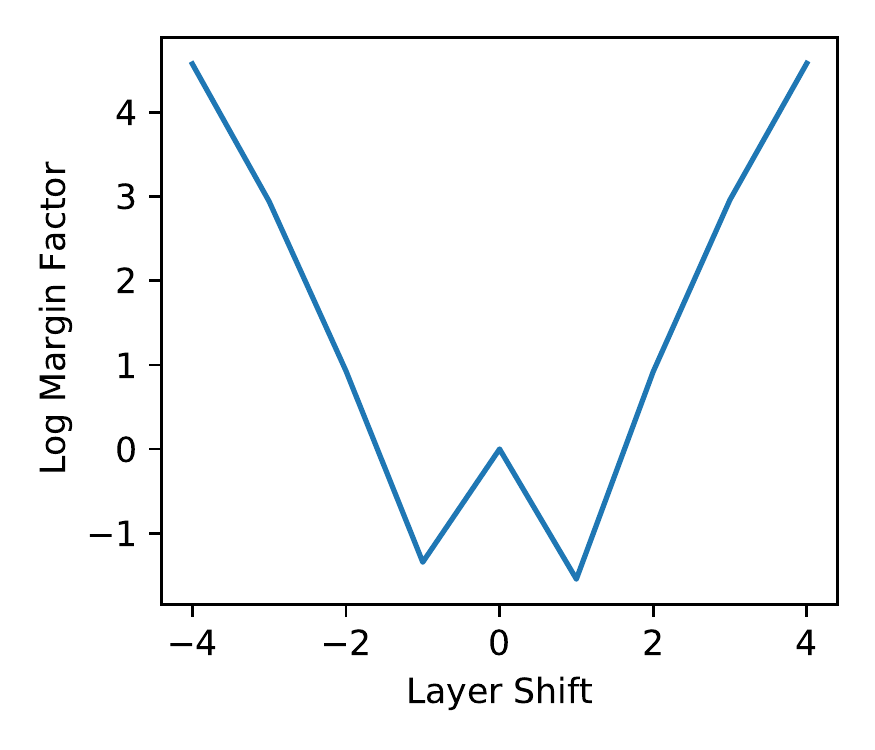}
    \label{appfig:Margin_IMD}
}
\caption{Matching neural network layers for MNIST across random seeds.
\label{appfig:margin}}
\end{figure}

To validate these results on a larger scale, we replicate the first experiment in \citet{kornblith2019similarity}. 
We train $10$ instances of an All-CNN-C convolutional NN \citep{springenberg2014striving} with $10$ layers (including the logits layer) on CIFAR-10, with different random initializations, reaching average train and test accuracies of $85\%$ and $80\%$, respectively.
We compute the test set embeddings of each layer in each NN instance, and evaluate the accuracy of finding corresponding layers in all other networks based on the minimal distance (maximum similarity) between them.
We treat each layer embedding as a separate dataset and compute distances between these datasets using LES, IMD and centered kernel alignment (CKA) \citep{kornblith2019similarity}, a similarity measure that is closely related to canonical correlation analysis and requires known correspondence between embeddings, i.e. the same input to all NN, unlike LES and IMD.
Table \ref{table:supp:layer_class} presents the layer classification accuracies for the compared distance and similarity measures in two scenarios, ``Identical inputs'', in which the layer embeddings of all NN were computed using the same $N=10000$ test samples, and ``Different Inputs'', in which the layer embeddings of each NN were computed based on a random subset of $N\in[7000,10000]$ test samples.
In the latter scenario, the correspondence between the different inputs is unknown, and therefore, CKA cannot be computed.
Table \ref{table:supp:layer_class} depicts that LES is capable of correctly identifying similarities of layer embeddings, obtaining accuracies that are close to CKA even with unknown input correspondence.

\begin{table}[t]
\caption{Layer position classification accuracy.}
\label{table:supp:layer_class}
\begin{center}
\begin{tabular}{lcc}
\toprule
{} &            Identical inputs &            Different inputs\\
\midrule
LES &           96\% &                        95.8\%\\
IMD &           85.2\% &                      81\%\\
CKA  &          97.3\% &                      -\\
\bottomrule
\end{tabular}
\end{center}
\end{table}

These experiments serve as a sanity-check for LES's applicability to such tasks and demonstrate that LES is successful at ensuring both that corresponding layers are assigned a small distance, and that distances between non-corresponding layers are large. 
The results highlight the potential of LES for analyzing and understanding different NN architectures, even when trained on different datasets.

\end{document}